\documentclass[preprint,12pt]{elsarticle} 

\usepackage{hyperref}
\usepackage{url}
\usepackage{subcaption}
\usepackage{amsthm,amsmath,amssymb}
\usepackage{mathrsfs}
\usepackage{booktabs}
\usepackage{comment}
\usepackage{apxproof}
\usepackage{thmtools, thm-restate}
\usepackage{wrapfig}
\usepackage{adjustbox}
\usepackage{graphicx}
\usepackage{dashrule}
\usepackage{algorithmic}
\usepackage{float}
\usepackage{ulem}
\usepackage{natbib}
\usepackage{microtype}
\usepackage{booktabs} 
\usepackage{cancel}
\usepackage[figuresright]{rotating}

\usepackage{array}
\usepackage{extarrows}
\usepackage{caption}
\usepackage{sidecap}
\usepackage{framed}
\usepackage{tabularx}
\usepackage{twoopt}
\usepackage{lscape}
\usepackage[table]{xcolor}
\usepackage{algorithm}
\usepackage{textcomp}
\usepackage{stfloats}
\usepackage{verbatim}

\usepackage{enumitem}

\usepackage{tikz}
\usetikzlibrary{fit}

\usepackage{lineno}

\newcommandtwoopt\Textbox[5][7.2cm][2cm]{%
\begin{tikzpicture}[remember picture,overlay]
  \coordinate (aux) at ([xshift=#1]#4);
  \node[inner ysep=3pt,yshift=1ex,draw=pink,thick,
    fit=(#3) (aux),baseline] 
    (box) {};
  \node[text width=#2,anchor=north east,
    font=\sffamily\footnotesize,
  align=right
    ] 
    at (box.north east) {#5};
\end{tikzpicture}%
}

\hypersetup{
colorlinks=true,
citecolor=red,
    linkcolor=red,
    filecolor=magenta,      
    urlcolor=red,
linktocpage}

\usepackage{lipsum}


\title{Safe Multi-Agent Reinforcement Learning with
Bilevel Optimization in Autonomous Driving}

\author{Zhi Zheng$^{*}$, Shangding Gu$^{*}$
}

\cortext[cor1]{Technical University of Munich,  Munich, 85748, Germany. We appreciate any constructive
comments and suggestions corresponding to \textit{zhi.zheng1997@gmail.com} and \textit{shangding.gu@tum.de}. } 

\begin{document}

\begin{frontmatter}

\begin{abstract}
Ensuring safety in MARL, particularly when deploying it in real-world applications such as autonomous driving, emerges as a critical challenge. To address this challenge, traditional safe MARL methods extend MARL approaches to incorporate safety considerations, aiming to minimize safety risk values. However, these safe MARL algorithms often fail to model other agents and lack convergence guarantees, particularly in dynamically complex environments. In this study, we propose a safe MARL method grounded in a Stackelberg model with bi-level optimization, for which convergence analysis is provided. Derived from our theoretical analysis, we develop two practical algorithms, namely Constrained Stackelberg Q-learning (CSQ) and Constrained Stackelberg Multi-Agent Deep Deterministic Policy Gradient (CS-MADDPG), designed to facilitate MARL decision-making in autonomous driving applications. To evaluate the effectiveness of our algorithms, we developed a safe MARL autonomous driving benchmark and conducted experiments on challenging autonomous driving scenarios, such as merges, roundabouts, intersections, and racetracks. The experimental results indicate that our algorithms, CSQ and CS-MADDPG, outperform several strong MARL baselines, such as Bi-AC, MACPO, and MAPPO-L, regarding reward and safety performance. The demos and source code are available at the link\footnote{\url{https://github.com/SafeRL-Lab/Safe-MARL-in-Autonomous-Driving.git}}.
\end{abstract}


\end{frontmatter}

\section{Introduction}
\label{sec:introduction}

Multi-Agent Reinforcement Learning (MARL) has shown remarkable success in many domains~\cite{gu2022review, zhang2021multi}. However, many MARL methods do not take safety into account~\cite{gu2022review}. In real-world scenarios, prioritizing safety is of paramount importance \cite{gu2023human, JMLR:v16:garcia15a}. For instance, in autonomous driving, a domain characterized by the interaction of multiple vehicles within a shared environment, where safety emerges as a critical concern~\cite{gu2022constrained}. Neglecting safety considerations can lead to hazardous situations with potentially catastrophic consequences.

To address this challenging problem, the field of safe MARL has emerged~\cite{cai2021safe, elsayed2021safe, gu2024safe, GU2023103905, liu2021cmix, lu2021decentralized}, focusing on the integration of safety into Reinforcement Learning (RL) algorithms. By explicitly prioritizing safety, these approaches aim to strike a delicate balance between optimizing performance and ensuring safe actions within multi-agent systems. For instance, Safe Dec-PG~\cite{lu2021decentralized} introduces a decentralized policy optimization method to ensure the safety of MARL. However, this method involves sharing policy parameters, which could potentially lead to suboptimal policy optimization~\cite{kuba2021settling}. Meanwhile, the MACPO and MAPPO-L~\cite{GU2023103905} methods utilize constrained multi-agent policy optimization and Lagrangian optimization to guarantee reward monotonic improvement while ensuring learning safety. Nonetheless, in these safe MARL algorithms, an agent often lacks effective modeling of other agents. Particularly in complex environments, the actions and policies of other agents can evolve over time, leading to non-stationarity. This evolving landscape poses significant challenges for learning optimal policies and effectively coordinating actions.

In this study, we proposed a safe MARL method with theoretical convergence analysis in a Stackelberg game setting~\cite{sherali1984multiple, zheng2022stackelberg}. The Stackelberg model provides a framework for hierarchical decision-making,  wherein agents assume distinct leader and follower roles. Specifically, the Stackelberg leader, functioning as the primary agent, has the capacity to both model and exert influence over the actions of the Stackelberg followers. This structure endows the leader with the capability to optimize its objectives while considering the responses elicited from the followers. By leveraging the Stackelberg model, agents can demonstrate more sophisticated strategies, improving coordination and enhancing overall performance~\cite{biac, huang2022robust, bimarkov, naveiro2019gradient}. These advancements have ignited new possibilities for MARL applications in real-world environments.

Based on the Stackelberg framework~\cite{sherali1984multiple, zheng2022stackelberg}, we leverage bi-level optimization to enhance the reward and safety performance of individual agents and provide a comprehensive convergence analysis. Building upon this theoretical foundation, we have developed two practical algorithms for multi-agent reinforcement learning (MARL) decision-making applications in autonomous driving: Constrained Stackelberg Q-learning (CSQ) for discrete action environments and Constrained Stackelberg Multi-Agent Deep Deterministic Policy Gradients (CS-MADDPG) for continuous action environments.

This study presents a four-fold contribution. First, we propose a safe MARL method wherein agents model each other, significantly enhancing cooperative performance while ensuring safety. Second, we provide a theoretical convergence analysis, demonstrating that agent rewards can converge to an optimal point while maintaining safety. Third, based on our theoretical analysis, we develop two algorithms and conduct experiments in autonomous driving. The experimental results clearly indicate that our algorithms significantly improve reward performance while ensuring agent safety in complex and dynamic autonomous driving environments. These results also showcase superior performance compared to other safe MARL baselines. Fourth, we develop safe MARL autonomous driving benchmarks to evaluate the effectiveness of safe MARL algorithms in autonomous driving applications. These benchmarks offer impartial evaluations for all safe MARL baselines, enhancing the reliability of our findings.

\section{Related Work}
\label{sec:related-work}
Safe RL often involves the application of Constrained Markov Decision Process (CMDP) frameworks \cite{altman-constrainedMDP}. CMDP is an extension of the standard Markov Decision Process (MDP) \cite{Uther2010}, incorporating constraints into the decision-making process. Numerous methodologies are proposed to address CMDP problems. For example, one notable approach is Reward Constrained Policy Optimization (RCPO), detailed in \cite{DBLP:journals/corr/abs-1805-11074}. RCPO employs the Lagrange relaxation technique to transform CMDPs into equivalent unconstrained problems, allowing for more effective optimization. Another well-established technique is Constrained Policy Optimization (CPO) \cite{CPO}, which builds upon the Trust Region method \cite{schulman2015trust}. In CPO, the agent's objective and constraint functions are treated as surrogate functions, simplifying the estimation process. It's worth noting that these algorithms have primarily been tailored for single-agent scenarios, with limited exploration into the realm of safe multi-agent learning.

In the domain of MARL, tackling safety concerns presents a notably more complex challenge than single-agent RL~\cite{gu2022review}. Each agent must navigate a delicate balance between adhering to its safety constraints and pursuing reward maximization, a juxtaposition that often leads to conflicting imperatives. Furthermore, each agent must integrate the safety constraints of other agents into its decision-making process. This consideration is vital to ensure that the actions of all agents within the system align with established safety bounds, thereby guaranteeing comprehensive, safe behaviors across the entire MARL system.

To enhance safety in MARL, two primary approaches, centralized shielding and factored shielding, have been developed by extending the concepts of single-agent safe RL shielding techniques~\cite{elsayed2021safe}. These strategies effectively adjust joint actions to mitigate unsafe behaviors. Nevertheless, their method necessitates external information for corrective measures, which might not always be feasible in practical scenarios. Melcer et al.~\cite{melcer2022shield} introduced a decentralized shield technique for MARL, employing a designated shield agent to manage and correct potentially unsafe actions of other agents within the system. Despite outperforming baselines in their study, this method encounters challenges similar to those identified in the work of Elsayed et al.~\cite{elsayed2021safe}.  Ding et al.~\cite{ding2023provably} proposed an approach to safe MARL within a zero-sum game framework, utilizing Lagrangian optimization to maintain system safety. While a rigorous theoretical framework underpins their method, it lacks empirical validation and may be ineffective for opponent modeling in bi-level scenarios. Zhang et al.~\cite{zhang2023spatial} offered a safety verification technique for MARL employing control barrier functions, demonstrating superior performance over baseline methods. However, this approach does not provide a convergence analysis and depends on external knowledge to define safe actions.

A policy optimization strategy incorporating peak and average constraints~\cite{liu2021cmix} is introduced, utilizing QMIX techniques~\cite{rashid2020monotonic} within a centralized training and decentralized execution paradigm to ensure MARL safety. This method surpasses traditional baselines, such as VDN~\cite{sunehag2018value} and IQL~\cite{tan1993multi}, in safety and reward performance within blocker game and vehicular network routing challenges. Nonetheless, it lacks theoretical guarantees and may only achieve suboptimality due to shared policy parameters~\cite{kuba2021settling}.  Lu et al.~\cite{lu2021decentralized} introduced a decentralized policy optimization technique for safe MARL. However, this approach might also lead to suboptimal policy outcomes due to the use of shared policy parameters~\cite{kuba2021settling}, highlighting a critical area for further investigation and development in the pursuit of optimal safety strategies within MARL environments.

Furthermore, two safety-aware model-free MARL algorithms~\cite{GU2023103905}, MACPO and MAPPO-Lagrangian, are introduced based on the adaptations of Trust Region and Lagrangian techniques to the multi-agent domain. These advancements guarantee performance enhancement and safety, marking a significant stride towards comprehensive safety improvement, and their experiment results demonstrate their algorithms outperform several state-of-the-art baselines such as HAPPO~\cite{kuba2022trust} and MAPPO~\cite{yu2022surprising}. Moreover, \cite{gu2024safe} proposes a safe MARL method for real-world robot control, where a soft policy optimization technique is leveraged to address the learning oscillations. Nevertheless, these algorithms often overlook inter-agent communication, with each agent primarily focused on optimizing its individual objective function. Moreover, the above methods lack convergence analysis. In this study, we ensure each agent's reward and safety performance with theoretical guarantees, provide convergence analysis, and develop safe MARL autonomous driving benchmarks for safe MARL communities.

\section{Preliminaries}
\label{sec:Preliminaries}
\subsection{Stackelberg model and bi-level reinforcement learning}
\subsubsection{Stackelberg model}
The Stackelberg model, originally proposed by Heinrich von Stackelberg~\cite{stackelberg}, stands as a pivotal concept in game theory, frequently employed to represent duopolistic interactions in the business domain. Unlike the simultaneous action selection in other models, the Stackelberg model embraces a sequential decision framework.

Within this model, two key players are the leader and the follower. The leader is granted the privilege of making the initial move by committing to an action. Subsequently, the follower devises an action in response to the leader's initial choice. An essential foundation of the model is the presumption that each player possesses knowledge or predictive capabilities, enabling them to anticipate their counterpart's response.

For example, the leader strategically engineers their initial move to maximize their payoff, carefully considering the follower's expected reaction. Significantly, the decision cannot be changed or reversed once the leader commits to their decision. Consequently, the follower's choices are limited to formulating their response to the leader's action to optimize their payoff. Once both players have finalized their strategies, their actions are set into motion simultaneously. The result achieved through the strategy is referred to as the Stackelberg equilibrium.

\subsubsection{Bi-level Actor Critic}
Many existing multi-agent reinforcement learning algorithms approach cooperative games from a symmetric perspective, treating each agent equally. However, this approach could result in the emergence of arbitrary Nash equilibria in Markov games, especially in scenarios with multiple equilibria. The learned behaviors of individual agents might suffer from uncertainty and suboptimality. To address this challenge, \textit{Zhang et al.}~\cite{biac} introduce an asymmetric Stackelberg model. In this model, they formulate the problem as an unconstrained Stackelberg problem\footnote{This problem is usually referred to as a bi-level reinforcement learning (BiRL) problem, which serves as an extension of the Stackelberg model into a multi-state environment~\cite{biac}.} within a two-agent reinforcement learning system. The preference for Stackelberg equilibrium over Nash equilibrium is rooted in its potential as a convergence point that offers superior Pareto optimality, particularly in cooperative settings.

In their setup, the leader agent aims to optimize its policy, trying to maximize its total return while considering the expected response of the follower agent. Simultaneously, the follower agent observes the leader's policy and selects its own optimal policy accordingly. Consequently, this reinforcement learning problem can be conceptualized as a bi-level optimization problem.
\begin{equation}
    \begin{aligned} \label{birl}
        \mathop{\max}_{\pi_1} \quad & J_1(\pi_1, \pi_2^*) = \mathbb{E}_{a_1,a_2\sim \pi_1,\pi_2^*} \sum_{t=1}^{\infty} \gamma^t r_1^t, \\
        \textrm{s.t.} \quad &  \pi_2^* = \mathop{\max}_{\pi_2} J_2(\pi_1, \pi_2) = \mathbb{E}_{a_1,a_2\sim \pi_1,\pi_2} \sum_{t=1}^{\infty} \gamma^t r_2^t.
    \end{aligned}
\end{equation}
In the equation above, $J_i$ represents the expected cumulative reward for agent $i$ under the conditions where both agents adhere to their respective policies, $\pi_1$ and $\pi_2$. 

However, the applicability of this method~\cite{biac} is limited to scenarios involving simple discrete actions and states.  Deploying their method in complex environments with discrete or continuous action spaces presents challenges due to the significant increase in computational load associated with the bi-level optimization problem during iterations and the lack of considerations for agent safety.

\subsection{Problem Formulation}
\label{sec:problem-formulation}

To enhance safety within the Stackelberg reinforcement learning framework, we integrate a set of constraints that manage the actions or policies of the agents, drawing upon the CMDP principles. These constraints are designed to align with safety criteria, resource limitations, or other domain-specific considerations. Beyond optimizing the primary objective, our framework obliges the agents to regulate one or more constraints in complex discrete and continuous environments. In our settings, the constraint is quantified as the cumulative long-term discounted costs associated with adherence to policies $\pi_1$ and $\pi_2$, as illustrated below:
\begin{equation}
\label{eq:constraints-bilevel}
    G_{i,j}(\pi_1,\pi_2) = \mathbb{E}_{a_1,a_2 \sim \pi_1, \pi_2}\sum_{t=1}^{\infty}\gamma^t c_{i,j}^t, \quad j=1,2...m.
\end{equation}
Here, $G_{i,j}$ represents the $j$-th constraint of Agent $i$, and $c_{i,j}^t$ denotes the $j$-th cost of Agent $i$ at time step t. By augmenting (\ref{birl}) with these constraints, we can formulate the constrained Stackelberg multi-agent reinforcement learning problem as follows:
\begin{equation}
    \begin{aligned} \label{constrained_birl}
        \mathop{\max}_{\pi_1} \quad & J_1(\pi_1, \pi_2^*) = \mathbb{E}_{a_1,a_2\sim \pi_1,\pi_2^*} \sum_{t=1}^{\infty} \gamma^t r_1^t,
        \\
        \textrm{s.t.} \quad & G_{1,j}(\pi_1,\pi_2^*) = \mathbb{E}_{a_1,a_2 \sim \pi_1, \pi_2^*}\sum_{t=1}^{\infty}\gamma^t c_{1,j}^t \leq d_{1,j},
        \\ 
        & \qquad \qquad \forall j=1...m_1,
        \\
        & \pi_2^* = \mathop{\max}_{\pi_2} \quad J_2(\pi_1, \pi_2) = \mathbb{E}_{a_1,a_2\sim \pi_1,\pi_2} \sum_{t=1}^{\infty} \gamma^t r_2^t,
        \\
        & \textrm{s.t.} \quad  G_{2,j}(\pi_1,\pi_2) = \mathbb{E}_{a_1,a_2 \sim \pi_1, \pi_2}\sum_{t=1}^{\infty}\gamma^t c_{2,j}^t \leq d_{2,j},
        \\
        & \qquad \qquad \qquad \forall j=1...m_2.
    \end{aligned}
\end{equation}
In our formulated problem, constraints can assume various forms, namely inequalities or equalities, depending on specific requirements. Generally, we focus on inequality constraints, where $d_{i,j}$ denotes the threshold limit for the $j$-th constraint of Agent $i$. 

\section{Method}
\label{sec:method}
\subsection{Constrained Stackelberg Q-learning}
The problem presented in (\ref{constrained_birl}) is intrinsically complex, primarily due to the intricate opponent-modeling aspects involved in constrained bi-level optimization. In response to this challenge, we introduce an approach called Constrained Stackelberg Q learning (CSQ). This approach is designed to tackle the intricacies of autonomous driving scenarios characterized by discrete action spaces and continuous state spaces, as outlined in Algorithm~\ref{algo:csq-algorithm-learning}. To simplify the notations of our method, it's important to note that the following discussion focuses on scenarios where each agent is subject to one constraint. Moreover, we analyze the convergence of our method, which is available in ~\ref{Convergence}.

\subsubsection{Value functions for reward and cost}
In our constrained Q-learning framework, the value function assumes a pivotal role, representing the expected cumulative reward attainable by an agent when following a particular policy within a specific state. However, constrained Q-learning introduces a crucial component—a cost-value function associated with actions and state transitions—beyond the reward aspect to enforce constraints. Both the reward and cost Q-functions are defined in Equations (\ref{q_value}) and (\ref{cost_value}), and they are estimated through neural networks. In this study, we utilize a centralized variant of the Q-function, which considers a global state or observation along with the actions of all agents. It is worth noting that each agent is responsible for managing its own Q-function networks.
\begin{align} \label{q_value}
    Q_i(s,a_1,a_2) = & \mathbb{E}_{s' \sim P}(r_i(s,a_1,a_2)  \nonumber \\& + \gamma\mathbb{E}_{a_1', a_2' \sim \pi_1,\pi_2}Q_i(s',a_1',a_2')),
\end{align}
\begin{align}\label{cost_value}
    G_{i}(s,a_1,a_2) = & \mathbb{E}_{s' \sim P}(c_{i}(s,a_1,a_2) \nonumber \\& + \gamma\mathbb{E}_{a_1', a_2' \sim \pi_1,\pi_2}G_{i}(s',a_1',a_2')).
\end{align}
$Q_i$ stands for  the Q-function associated with the reward parameterized by $\phi_i$ and $G_{i}$ for cost with parameter $\zeta_{i}$. $P$ denotes the state transition distribution when following $\pi_1$ and $\pi_2$. The optimal policy in this context entails adhering to the equilibrium point of the constrained Stackelberg model, which will be illustrated below.
\subsubsection{Bellman equation}
Similar to the min-max-Q and Nash-Q approaches~\cite{nash_q}, we can define the constrained Stackelberg-Bellman equations for reward and cost, as illustrated below.:
\begin{equation} \label{p10}
    \begin{aligned}
    Q_i^{t+1}(s,a_1, a_2)  = &(1-\alpha)Q_i^{t}(s,a_1, a_2)  
    \\
    &+ \alpha(r_i^t + \gamma Q_i^{t,{CS}}(s')),
    \end{aligned}
\end{equation}

\begin{equation}\label{p11}
    \begin{aligned}
    G_{i}^{t+1}(s,a_1, a_2)  = &(1-\alpha)G_{i}^{t}(s,a_1, a_2) 
    \\
    &+ \alpha(c_{i}^t + \gamma G_{i}^{t,{CS}}(s')),
    \end{aligned}
\end{equation}
\begin{equation}\label{p12}
    Q_i^{t,{CS}}(s')= \pi_1(s')\pi_2(s')Q_i^t(s'),
\end{equation}
\begin{equation}\label{p13}
   G_{i}^{t,{CS}}(s')= \pi_1(s')\pi_2(s')G_{i}^t(s').
\end{equation}
Here, $Q_i^{t,{CS}}(s')$ and $ G_i^{t,{CS}}(s')$ denote the anticipated reward and cost, respectively, for the $i$-th agent within the state denoted as $s$ in the constrained Stackelberg equilibrium. These equations guide the iterative update of the Q-functions. 
\subsubsection{Update rule for value function}
To proceed with a transition $\langle s, a_1, a_2, r_1, r_2, c_1, c_2, d, s' \rangle$, the first step entails determining the Stackelberg equilibrium for the subsequent state $s'$. This is achieved by solving the following problem by enumerating all potential action combinations.
\begin{equation}
    \begin{aligned} \label{p14}
         a'_1  \gets &\mathop{\arg\max}_{a_1 \in A_1(s')} Q_1(s', a_1, \mathop{\arg\max}_{a_2 \in A_2(s')}  Q_2(s', a_1, a_2)),
        \\
          \quad \textrm{s.t.} &\quad G_1(s', a_1, a_2) \leq d_1.
        \\
         & \quad a'_2  \gets \mathop{\arg\max}_{a_2 \in A_2(s')} Q_2(s', a'_1, a_2),
        \\
        & \quad \quad \textrm{s.t.} \quad G_2(s', a'_1, a_2) \leq d_2.
    \end{aligned}
\end{equation}
$A_i(s')$ denotes the safe action set for Agent $i$ in state $s'$, with $d_i$ signifying Agent $i$'s constraint threshold. Consequently, actions derived from solving this equation are guaranteed to adhere to safety constraints. The parameters $\phi_i$ and $\zeta_i$ are subsequently updated via gradient descent to minimize the TD error for their respective Q-functions.
\begin{subequations}
\begin{equation}
    \begin{aligned}
        & \phi_1 \gets \phi_1 - \alpha _1 \bigtriangledown_{\phi_1}(Q_1(s,a_1,a_2) - r_1
        \\
        & \qquad - \gamma (1-d)Q_1^{targ}(s',a_1',a_2')),
    \end{aligned}
\end{equation}
\begin{equation}
    \begin{aligned}
        & \phi_2 \gets \phi_2 - \alpha_2 \bigtriangledown_{\phi_2}(Q_2(s,a_1,a_2)  - r_2
        \\
        & \qquad - \gamma (1-d)Q_2^{targ}(s',a_1',a_2')),
    \end{aligned}
\end{equation}

\begin{equation}
    \begin{aligned}
        & \zeta_1 \gets \zeta_1 - \beta_1 \bigtriangledown_{\zeta_1}(G_{1}(s,a_1,a_2) - c_1
        \\
        & \qquad - \gamma (1-d)G_{1}^{targ}(s',a_1',a_2')),
    \end{aligned}
\end{equation}

\begin{equation}
    \begin{aligned}
        & \zeta_2 \gets \zeta_2 - \beta_2 \bigtriangledown_{\zeta_2}(G_{2}(s,a_1,a_2) - c_2 
        \\
        & \qquad- \gamma (1-d)G_{2}^{targ}(s',a_1',a_2')).
    \end{aligned}
\end{equation}

\end{subequations}
where $\alpha_i$ and $\beta_i$ are the learning rate.

\begin{algorithm}[tb!]
    \caption{Constrained Stackelberg Q learning (CSQ)}
    \begin{algorithmic}
        \STATE \textbf{Initialization} 
        \STATE Initialize Q-function parameters $\phi_i$, cost value function parameters $\zeta_i$ and replay buffer $D$;
        \STATE Initialize target networks parameters, and they share the same parameters as main networks $\phi_i^{targ} = \phi_i$, $\zeta_i^{targ} = \zeta_i$, $i\in \left\{1,2\right\}$; 
        \STATE Initialize hyperparameters $d_i$, $\gamma$ and $\rho$, $i\in \left\{1,2\right\}$;
        \FOR{step=1:t}
            \STATE Sample actions with $\epsilon$-greedy according to Equation (\ref{p14});
            \STATE $a^*_1 \gets \mathop{\arg\max}_{a_1} Q_1(s, a_1, \mathop{\arg\max}_{a_2}  Q_2(s, a_1, a_2))$,
            \STATE $\quad \textrm{s.t.} \quad G_1(s, a_1, a_2) \leq d_1 $.
            \STATE $a^*_2 \gets \mathop{\arg\max}_{a_2} Q_2(s, a^*_1, a_2)$,
            \STATE $\quad \textrm{s.t.} \quad G_2(s, a^*_1, a_2) \leq d_2 $.
            \STATE Observe next state $s'$, reward $r_i$, cost $c_i$, and done signal $d$ which indicates whether $s'$ is terminal;
            \STATE Store trajectory $(s,a_i,r_i,c_i,d,s')$ into replay buffer $D$;
            \STATE If $s'$ is terminal, reset environment state;
        
            \FOR{k=1:n}
            \STATE Randomly sample a batch of transitions, $B = \left\{(s, a_i,  r_i , c_i, d,  s')\right\}$ from $D$;
            \STATE Compute target actions according to Equation (\ref{p14});
            \STATE Compute target Q:
            \STATE $y_{i} = r_{i} + \gamma(1-d)Q^{targ}_{i}(s', a'_1, a'_2)$.
            \STATE Compute target cost:
            \STATE $g_{i} = c_{i} + \gamma(1-d)G^{targ}_{i}(s', a'_1, a'_2)$.
            \STATE Update critic function by gradient descent:
            \STATE $\nabla_{\phi}\frac{1}{|B|}\sum_{\left\{(s ,a ,r ,c ,s', d) \in B\right\}}(Q_{i}(s, a_1, a_2) - y_{i})^2$.
            \STATE Update critic cost function by gradient descent:
            \STATE $\nabla_{\zeta}\frac{1}{|B|}\sum_{\left\{(s ,a ,r ,c ,s', d) \in B\right\}}(G_{i}(s, a_1, a_2) - g_{i})^2$.
            \STATE Update target network:
            \STATE $\phi_i^{targ} \gets \rho\phi_i^{targ} + (1-\rho)\phi_i$.
            \STATE $\zeta_i^{targ} \gets \rho\zeta_i^{targ} + (1-\rho)\zeta_i$.
            \ENDFOR
        \ENDFOR
    \end{algorithmic}
\label{algo:csq-algorithm-learning}
\end{algorithm}

\subsection{Constrained Stackelberg MADDPG}
Constrained Stackelberg Q-learning performs admirably in scenarios characterized by discrete action spaces. However, it encounters limitations as the number of possible actions increases, potentially causing an exponential rise in computational time. To tackle this challenge, we introduce the Constrained Stackelberg Multi-Agent Deep Deterministic Policy Gradients (CS-MADDPG) algorithm, as detailed in Algorithm~\ref{algo:CS-MADDPG-algorithm}. This algorithm provides an efficient solution by adapting the MADDPG algorithm~\cite{MADDPG} to the constrained Stackelberg model.

\subsubsection{Actor, Critic and Cost Critic}
MADDPG employs an actor-critic architecture for each agent. In this framework, the critic is responsible for assessing the expected return, considering the joint actions and observations of all agents, while the actor's role is to select actions based on the agent's local observations.

To align MADDPG with the Stackelberg model, we designate one agent as the leader (Agent 1) and the other as the follower (Agent 2). This distinction establishes a hierarchical relationship between the two agents, capturing the strategic decision-making dynamics inherent in the Stackelberg model. The leader agent's actor operates independently, relying solely on its local observations to make decisions. Subsequently, the follower agent observes the leader's decision and then formulates its decision accordingly. Both agents possess deterministic policies $\mu_i$ approximated by neural network parameters $\theta_i$. The hierarchical decision-making sequence is defined as follows.
\begin{subequations}
\label{hierarchical decision-making}
    \begin{align}
        & a_1 = \mu_1(o_1|\theta_1),
        \\
        & a_2 = \mu_2(o_2,a_1|\theta_2).
    \end{align}
\end{subequations}
Here, $o_i$ represents the local observation of each agent, and the decisions are made sequentially but executed simultaneously within a decentralized architecture.

The critic setting is similar to MADDPG~\cite{MADDPG}. It refers to the centralized Q-value function, which estimates the expected accumulated reward for a particular set of joint actions within a given state. This Q-value function guides the agent in selecting actions that maximize rewards. The critic network for Agent $i$ can be represented as follows:
\begin{equation}
    Q_i(s, a_1, a_2|\phi_i).
\end{equation}
where $s$ represents the global state and $\phi_i$ is the parameter.

In addition to the reward actor and critic, the cost critic plays a crucial role in evaluating policy safety. This cost critic serves as a centralized value function, estimating the expected accumulated cost or penalty associated with a specific set of joint actions in a given state, similar to the role of the reward critic. It plays a pivotal role in guiding the agent away from actions expected to result in higher costs or penalties. In this context, the cost critic network for Agent $i$'s $j$-th constraint is denoted as follows:
\begin{equation}
    G_{i,j}(s, a_1, a_2|\zeta_{i,j}), \quad j=1...m_i.
\end{equation}

\subsubsection{Problem reformulation and Lagrangian Approach}
To provide a more straightforward presentation of the specific problems we need to address, based on the aforementioned new definition, we can rewrite the constrained bi-level optimization problem (\ref{constrained_birl}) in three steps.

First, assuming that we have already well-approximated value and cost functions, the decision-making process at each time step involves solving the following constrained bi-level problem to obtain a set of optimal joint actions.
\begin{equation}
    \begin{aligned} \label{p19}
        & a_1^* = \mathop{\arg\max}_{a_1 \in A_1(s')} Q_1(s, a_1, a_2^*),
        \\
        & \ \ \textrm{s.t.} \quad G_{1,j}(s, a_1, a_2^*) \leq d_{1,j}, \quad \forall j=1...m_1,
        \\
        &\quad \qquad a_2^* = \mathop{\arg\max}_{a_2 \in A_2(s')} Q_2(s, a_1, a_2),
        \\
        & \quad \qquad \quad \textrm{s.t.} \quad G_{2,j}(s, a_1, a_2) \leq d_{2,j}, \quad \forall j=1...m_2.
    \end{aligned}
\end{equation}
However, solving such an optimization problem becomes impractical due to the immense computational effort it demands. This is precisely why we must employ the actor, which only requires forward propagation to derive actions. To accomplish this, we must integrate the actors' parameters into (\ref{p19}) and reformulate it as an optimization problem concerning the parameter $\theta_i$, as illustrated below.
\begin{equation}
    \begin{aligned} \label{p20}
        & \theta_1^* = \mathop{\arg\max}_{\theta_1} Q_1(s, \mu_1(o_1|{\theta_1}), \mu_2(o_2, \mu_1(o_1|{\theta_1})|{\theta_2^*})),
        \\
        & \textrm{s.t.} \quad G_{1,j}(s, \mu_1(o_1|{\theta_1}), \mu_2(o_2, \mu_1(o_1|{\theta_1})|{\theta_2^*}) \leq d_{1,j},
        \\
        &  \quad \qquad \qquad \forall j=1...m_1,
        \\
        & \qquad \theta_2^* = \mathop{\arg\max}_{\theta_2} Q_2(s, \mu_1(o_1|{\theta_1}), \mu_2(o_2, \mu_1(o_1|{\theta_1})|{\theta_2}),
        \\
        & \qquad \quad \textrm{s.t.} \quad G_{2,j}(s, \mu_1(o_1|{\theta_1}), \mu_2(o_2, \mu_1(o_1|{\theta_1})|{\theta_2}) \leq d_{2,j},
        \\
        & \qquad \qquad \qquad \qquad  \forall j=1...m_2.
    \end{aligned}
\end{equation}
\\
Furthermore, we employ a Lagrangian relaxation approach to tackle the issue outlined in problem (\ref{p20}), yielding an unconstrained bi-level problem:
\begin{equation}
    \begin{aligned} \label{p21}
        &\mathop{\max}_{\theta_1}\mathop{\min}_{\lambda_1 \geq 0} L_1(\theta_1,\lambda_1) = 
        \\
        &Q_1(s, \mu_1(o_1|{\theta_1}), \mu_2(o_2, \mu_1(o_1|{\theta_1})|{\theta_2^*}))
        \\
        & -\sum_{j=0}^{m_1}\lambda_{1,j}(G_{1,j}(s, \mu_1(o_1|{\theta_1}), \mu_2(o_2, \mu_1(o_1|{\theta_1})|{\theta_2^*}) - d_{1,j}),
        \\
        & \qquad\textrm{s.t.} \mathop{\max}_{\theta_2}\mathop{\min}_{\lambda_2 \geq 0} L_2(\theta_2,\lambda_2) = 
        \\
        & \qquad Q_2(s, \mu_1(o_1|{\theta_1}), \mu_2(o_2, \mu_1(o_1|{\theta_1})|{\theta_2}))
        \\
        & \quad - \sum_{j=0}^{m_2}\lambda_{2,j}(G_{2,j}(s, \mu_1(o_1|{\theta_1}), \mu_2(o_2, \mu_1(o_1|{\theta_1})|{\theta_2}) - d_{2,j}).
    \end{aligned}
\end{equation}
where $\lambda_{i,j}$ denotes the Lagrange multipliers, implying adopting a two-timescale approach. At a faster timescale, the optimal values of $\theta$ are determined by solving the problem (\ref{p21}). Meanwhile, the Lagrange multiplier $\lambda$ gradually increases at a slower timescale until the constraints are successfully met.

\subsubsection{Actor update}
In our reinforcement learning environment, we consider only two agents, each with one constraint, i.e., $m_1=1,m_2=1$. Using the abovementioned equations, we can derive the update rules for the actors' parameters and Lagrange multipliers.
\begin{subequations}
    \begin{align} 
       & \theta_1 \gets \theta_1 + \alpha_1 \bigtriangledown_{\theta_1}L_1(\theta_1,\lambda_1), \label{p22a}\\
       & \lambda_1 \gets \lambda_1 - \beta_1 \bigtriangledown_{\lambda_1}L_1(\theta_1,\lambda_1), \label{p22b}
       \\ 
       & \theta_2 \gets \theta_2 + \alpha_2 \bigtriangledown_{\theta_2}L_2(\theta_2,\lambda_2), \label{p22c}\\
       & \lambda_2 \gets \lambda_2 - \beta_2 \bigtriangledown_{\lambda_2}L_2(\theta_2,\lambda_2). \label{p22d}
    \end{align} 
\end{subequations}
where $\alpha_i$ and $\beta_i$ represent the learning rates for $\theta_i$ and $\lambda_i$, respectively. The derivatives in (\ref{p22a},\ref{p22b},\ref{p22c},\ref{p22d}) with respect to $\theta_i$ and $\lambda_i$ can be calculated using a set of trajectory samples $B$ randomly selected from the replay buffer, and these derivatives can be computed through the application of the chain rule.
\\
\textbf{Leader:}
\begin{equation}
    \begin{aligned} \label{p23}        \bigtriangledown_{\theta_1}L_1(\theta_1,\lambda_1) = & \mathbb{E}_{\tau \sim B} \bigtriangledown_{\theta_1}\mu_1(o_1)(\bigtriangledown_{a_1}Q_1(s,a_1,a_2)+
    \\
    & \bigtriangledown_{a_1}\mu_2(o_2,a_1)\bigtriangledown_{a_2}Q_1(s,a_1,a_2) - \\ & \sum_{j=0}^{m_1}\lambda_{1,j}\bigtriangledown_{a_1}G_{1,j}(s,a_1,a_2)-\\
    &\bigtriangledown_{a_1}\mu_2(o_2,a_1)\bigtriangledown_{a_2}G_{1,j}(s,a_1,a_2)). 
    \end{aligned}
\end{equation}

\begin{equation}\label{p24}
    \bigtriangledown_{\lambda_1}L_1(\theta_1,\lambda_1) = -\mathbb{E}_{\tau
    \sim B} \sum_{j=0}^{m_1}\lambda_{1,j}(G_{1,j}(s, a_1, a_2)) - d_{1,j}),
\end{equation}
\\
\textbf{Follower:}
\begin{equation}
    \begin{aligned}\label{p25}
        \bigtriangledown_{\theta_2}L_2(\theta_2,\lambda_2) = \mathbb{E}_{\tau
        \sim B}\bigtriangledown_{\theta_2}\mu_2(o_2,a_1)(\bigtriangledown_{a_2}Q_2(s,a_1,a_2) - 
        \\
        \sum_{j=0}^{m_2}\lambda_{2,j}\bigtriangledown_{a_2}G_{2,j}(s,a_1,a_2)),
    \end{aligned}
\end{equation}

\begin{equation}
    \begin{aligned}\label{p26}
        \bigtriangledown_{\lambda_2}L_2(\theta_2,\lambda_2) = -\mathbb{E}_{\tau
        \sim B} \sum_{j=0}^{m_2}\lambda_{2,j}(G_{2,j}(s, a_1, a_2)) - d_{2,j}).
    \end{aligned}
\end{equation}

\subsubsection{Critic and Cost Critic update}
The centralized value function and cost value function $Q_i$ and $G_{i,j}$ are updated using off-policy temporal difference learning, as defined in MADDPG. The update rule closely resembles that of MADDPG. However, a hierarchical decision-making process determines the agents' target actions (\ref{hierarchical decision-making}). These networks are updated by minimizing the temporal difference error $L$ via gradient descent.
\\
\textbf{Critic:}
\begin{equation}
    \begin{aligned}
        L(\phi_i) = &\mathbb{E}_{\tau \sim B}(Q_i(s,a_1,a_2) - y)^2,
        \\
        y = &r_i + (1-d)\gamma Q_i^{targ}(s',a_1',a_2'),
        \\
        a_1' = &\mu_1^{targ}(o_1'),
        \\
        a_2' = &\mu_2^{targ}(o_2',a_1').
    \end{aligned}
\end{equation}
\textbf{Cost Critic:}
\begin{equation}
    \begin{aligned}
        L(\zeta_{i,j}) = &\mathbb{E}_{\tau \sim B}(G_{i,j}(s,a_1,a_2) - y)^2,
        \\
        y = &c_i + (1-d)\gamma G_{i,j}^{targ}(s',a_1',a_2'),
        \\
        a_1' = &\mu_1^{targ}(o_1'), 
        \\
        a_2' = &\mu_2^{targ}(o_2',a_1'). 
    \end{aligned}
\end{equation}
Where $\mu_i^{targ}$ is the target policy with delayed parameters $\theta_i^{targ}$, and $Q_i^{targ}$, $G_{i,j}^{targ}$ are the target critics, and target cost critics with delayed parameters $\phi_i^{targ}$ and $\zeta_{i,j}^{targ}$. These target networks are periodically updated with the latest parameters, contributing to the stabilization of the learning process.

\begin{algorithm}[htbp]
    \caption{Constrained Stackelberg MADDPG (CS-MADDPG)}
    \begin{algorithmic}
        \STATE \textbf{Initialization} 
        \STATE Initialize policy parameters $\theta_i$, Q-function parameters $\phi_i$, cost value function parameters $\zeta_i$ and empty replay buffer $D$;
        \STATE Initialize target networks parameters, and they share the same parameters as main networks $\theta_i^{targ} = \theta_i$, $\phi_i^{targ} = \phi_i$, $\zeta_i^{targ} = \zeta_i$, $i\in \left\{1,2\right\}$;
        \STATE Initialize hyperparameters $d_i$, $\gamma$, $\alpha_i$, $\beta_i$ and $\rho$ , $i\in \left\{1,2\right\}$;
        \FOR{step=1:t}
            \STATE Observe state s  and select action:
            \STATE $ \quad a_1 = clip(\mu_1(s) + \epsilon, a_{low}, a_{high})$,
            \STATE $ \quad a_2 = clip(\mu_2(s, a_1) + \epsilon, a_{low}, a_{high})$, where $\epsilon \sim \mathbb{N}$.
            \STATE Execute $(a_1,a_2)$ in the environment;
            \STATE Observe next state $s'$, reward $r_i$, cost $c_i$, and done signal $d$ which indicates whether $s'$ is terminal;
            \STATE Store trajectory $(s,a_i,r_i,c_i,d,s')$ into replay buffer $D$;
            \STATE If $s'$ is terminal, reset environment state;
            \FOR{k=1:n}
                \STATE Randomly sample a batch of transitions, $B = \left\{(s, a_i,  r_i , c_i, d,  s')\right\}$ from $D$;
                \STATE Compute target actions:
                \STATE $a'_1 = \mu_1(s)$, 
                \STATE $a'_2 = \mu_2(s,a'_1)$.
                \STATE Compute target Q:
                \STATE $y_{i} = r_{i} + \gamma(1-d)Q^{targ}_{i}(s', a'_1, a'_2)$.
                \STATE Compute target cost:
                \STATE $g_{i} = c_{i} + \gamma(1-d)G^{targ}_{i}(s', a'_1, a'_2)$.
                \STATE Update critic function by gradient descent:
                \STATE $\nabla_{\phi}\frac{1}{|B|}\sum_{\left\{(s ,a ,r ,c ,s', d) \in B\right\}}(Q_{i}(s, a_1, a_2) - y_{i})^2$.
                \STATE Update critic cost function by gradient descent:
                \STATE $\nabla_{\zeta}\frac{1}{|B|}\sum_{\left\{(s ,a ,r ,c ,s', d) \in B\right\}}(G_{i}(s, a_1, a_2) - g_{i})^2$.
                \STATE Update leader and follower's actor using policy gradient according to Equations (\ref{p23}) and (\ref{p25}):
                \STATE $\theta_i^{k+1} = \theta_i^k + \alpha_i \bigtriangledown_{\theta_i}L_i(\theta_i,\lambda_i)$.
                \STATE Update leader and follower's Lagrange multiplier according to Equations (\ref{p24}) and (\ref{p26}):
                \STATE $\lambda_i^{k+1} = \lambda_i^k  + \beta_i \bigtriangledown_{\lambda}L_i(\theta_i,\lambda_i)$.
                \STATE Update target network:
                \STATE $\theta_i^{targ} \gets \rho\theta_i^{targ} + (1-\rho)\theta_i$,
                \STATE $\phi_i^{targ} \gets \rho\phi_i^{targ} + (1-\rho)\phi_i$,
                \STATE $\zeta_i^{targ} \gets \rho\zeta_i^{targ} + (1-\rho)\zeta_i$.
            \ENDFOR
        \ENDFOR
    \end{algorithmic}
    \label{algo:CS-MADDPG-algorithm}
\end{algorithm}

\subsection{Expanding Bi-Agent Systems into Multi-Agent Systems}
In the previous setup, constrained by the bi-level framework, we only considered scenarios involving two agents. However, theoretically, this framework can be extended to multi-agent scenarios. There are two potential expansions: horizontal expansion and vertical expansion.
\subsubsection{Horizontal Expansion}
In horizontal expansion, the bi-level optimization framework can be maintained while accommodating multiple leaders and followers concurrently. A relationship of mutual constraints persists between leaders and followers. However, no restrictions are imposed between leaders and leaders or between followers and followers. Consequently, agents on the same level engage in a Nash equilibrium competition. The decision-making sequence of these $n$ agents can also be illustrated as follows:
\begin{subequations}
    \begin{align}
        & a^l_1 = \mu^l_1(o^l_1), \quad a^l_2 = \mu^l_2(o^l_2), \ldots, \quad a^l_n = \mu^l_n(o^l_n),
        \\
        & a^f_1 = \mu^f_1(o^f_1, a^l_1, a^l_2,...,a^l_n),\ldots, \quad a^f_n = \mu^f_n(o^f_n, a^l_1, a^l_2,...,a^l_n).
    \end{align}
\end{subequations}
The superscript $l$ here represents the leader agent, and $f$ represents the follower agent.

\subsubsection{Vertical Expansion}
In this type of expansion, one can endeavor to extend the bi-level problem to a multi-level problem. The distinction between leaders and followers is no longer absolute but relative. This entails a deepening of the bi-level optimization problem. In the context of reinforcement learning, the multi-level optimization problem is presented in the formula below:
\begin{equation}
    \begin{aligned}
        \mathop{\max}_{\pi_1} \quad & J_1(\pi_1, \pi_2^*,..., \pi_n^*) = \mathbb{E}_{a \sim \pi} \sum_{t=1}^{\infty} \gamma^t r_1^t, \\
         \textrm{s.t.}  \quad &  \mathop{\max}_{\pi_2} \quad J_2(\pi_1, \pi_2,..., \pi_n^*) = \mathbb{E}_{a \sim \pi} \sum_{t=1}^{\infty} \gamma^t r_2^t,
        \\
        & \textrm{...}
        \\
        & \textrm{s.t.} \quad \mathop{\max}_{\pi_n} \quad J_n(\pi_1, \pi_2,..., \pi_n) = \mathbb{E}_{a \sim \pi} \sum_{t=1}^{\infty} \gamma^t r_n^t.  
    \end{aligned}
\end{equation}
For the sake of brevity, only the unconstrained optimization problem is shown in the following formula. The decision-making sequence of these $n$ agents can also be demonstrated as follows:
\begin{subequations}
    \begin{align}
        & a_1 = \mu_1(o_1),
        \\
        & a_2 = \mu_2(o_2,a_1),
        \\
        & a_3 = \mu_3(o_3,a_1,a_2),
        \\
        & ...  \nonumber
        \\
        & a_n = \mu_n(o_3,a_1,a_2,...,a_{n-1}).
    \end{align}
\end{subequations}

\section{Experiment}
\label{sec:experiment}
This section presents the experimental evaluation of our algorithms, CSQ and CS-MADDPG, in the safe MARL autonomous driving benchmarks. These benchmarks have been developed based on the highway-env environments~\cite{highway-env},  which uses the Gym toolkit developed by OpenAI \cite{brockman2016openai} and incorporates vehicle dynamics through the Kinematic Bicycle Model \cite{vehicelemodel}. Our objective is to thoroughly assess the safety and reward performance of these algorithms across a wide range of autonomous driving scenarios, including merge, intersection, roundabout, and racetrack situations. The safe MARL benchmarks provide comprehensive simulation environments tailored explicitly for autonomous driving.

Our experiment is primarily divided into two parts. The first part encompasses scenarios (Merge and Roundabout environments) employing a discrete action space, featuring actions such as 'LANE-LEFT,' 'IDLE,' 'LANE-RIGHT,' 'FASTER,' and 'SLOWER.' The second part involves scenarios (Intersection and Racetrack environments) utilizing a continuous action space, comprising the throttle $\alpha$ and steering angle $\delta$. In the discrete action experiment, we compare our algorithm, CSQ, and the baseline algorithm, Bi-AC. In the continuous action experiment, our algorithm CS-MADDPG is assessed against three other algorithms: Bi-AC, MACPO, and MAPPO-L \cite{GU2023103905}.

In the experiments, each agent earns a reward of 2 for maintaining a speed within the range of $[8,12]$. An additional reward of 10 is granted to the first agent who crosses the finish line, while the second agent receives 5. In the event of a collision, the colliding vehicle incurs a penalty cost of 5.

\subsection{Safe Merge Environments}
The environment consists of two main roads and a side road merging into one of the main roads, as illustrated in Fig \ref{fig:merge}. The leader agent starts on a main road adjacent to the side road, heading towards the junction. In contrast, the follower agent initiates on the side road and must merge into the main roads, avoiding collision with the obstacle at the side road's end. A finish line is placed after the road junction. Both vehicles' initial positions and velocities are preset with random noise added, potentially leading to a collision without significant adjustments. The leader's goal is to maintain a high speed, creating space for the follower to merge safely. The follower aims to merge successfully without colliding with the leader or the roadblock. This task involves a general-sum, mixed cooperative-competitive game, requiring cooperation for safe merging and competition for bonus rewards by reaching the finish line first.

In Table \ref{tab:merge_reward_label}, all reward curves for the merge environment are depicted, including the leader agent's reward, the follower agent's reward, and the total reward. Remarkably, our CSQ algorithm outperforms Bi-AC regarding reward once convergence is reached. Additionally, the reward curve trained by CSQ is more stable after convergence, with smaller fluctuations. The training curve in Fig \ref{Fig8.5} depicts the probability of various events occurring during training. Because both algorithms use a bi-level framework, the leader has a clear advantage over the follower. In most cases, the leader will make space for the follower to merge while accelerating. However, CSQ guarantees 100\% safety after convergence, while Bi-AC still shows an approximate 10\% chance of collision. 

\begin{figure}[tb]
\centering
\includegraphics[scale=0.6]{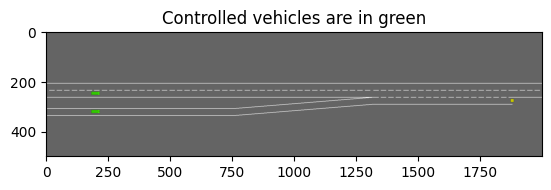}
\caption{Safe Merge Environments: The two vehicles need to merge into the main road safely.}
\label{fig:merge}
\end{figure}

\begin{figure}[tb]
\centering
\includegraphics[scale=0.6]{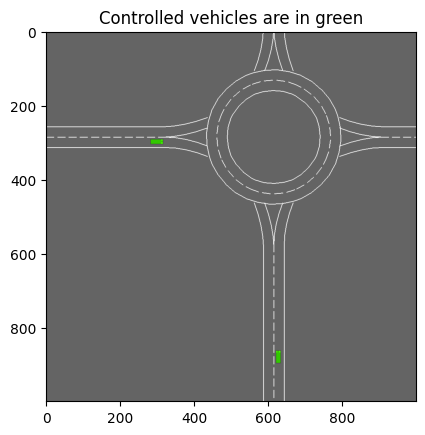}
\caption{Safe Roundabout Environments: 
Two vehicles need to merge safely into the roundabout and exit from the same exit.}
\label{fig:roundabout}
\end{figure}

\subsection{Safe Roundabout Environments}
This scenario depicts a circular intersection featuring a roundabout composed of two lanes, as illustrated in Fig \ref{fig:roundabout}. The roundabout has four entrances/exits in the ring's north, south, east, and west directions. In this setup, the leader agent vehicle originates from the west entrance, while the follower agent vehicle begins at the south entrance. These vehicles are on a collision course if left unattended due to their current speeds and lanes. Both vehicles autonomously follow their pre-planned routes, leading them towards the northern exit. However, they must also execute lane changes and longitudinal control to navigate the roundabout efficiently, minimizing the risk of collisions. This task embodies a general-sum and mixed cooperative-competitive game, where the two cars must collaborate to safely navigate the roundabout while seeking opportunities to prioritize their exits and attain higher rewards.

Table \ref{tab:roundabout_reward_label} and Fig \ref{Fig8.6} display the roundabout environment's reward curves and training progress. The training outcomes closely resemble those of the merge environment, showcasing that our CSQ algorithm consistently attains superior rewards and safety rates compared to Bi-AC. And CSQ also shows better stability after convergence. The experiment results in this environment show that the leader often accelerates to occupy the shorter inner lane, forcing the follower to drive on the outer lane. For a visualized illustration, please see our demos at this link: \url{https://github.com/SafeRL-Lab/Safe-MARL-in-Autonomous-Driving}.

\begin{figure}[htbp]
\centering
\includegraphics[scale=0.5]{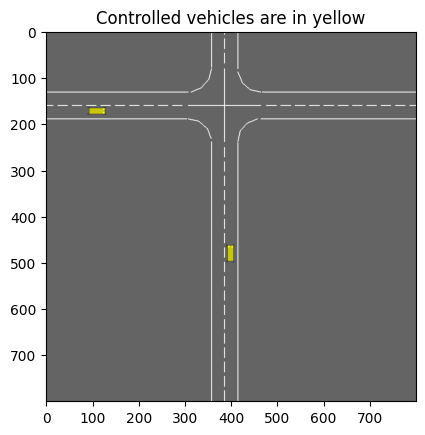}
\caption{Safe Intersection Environments: The two vehicles need to proceed straight and safely through the intersection.}
\label{fig:intersection}
\end{figure}

\begin{table*}[]
    \centering
    \begin{tabular}{ c  c  c }
    \hline \\
    Leader & Follower & Total \\
    \hline \\
    \includegraphics[width=0.3\textwidth,height=4cm]{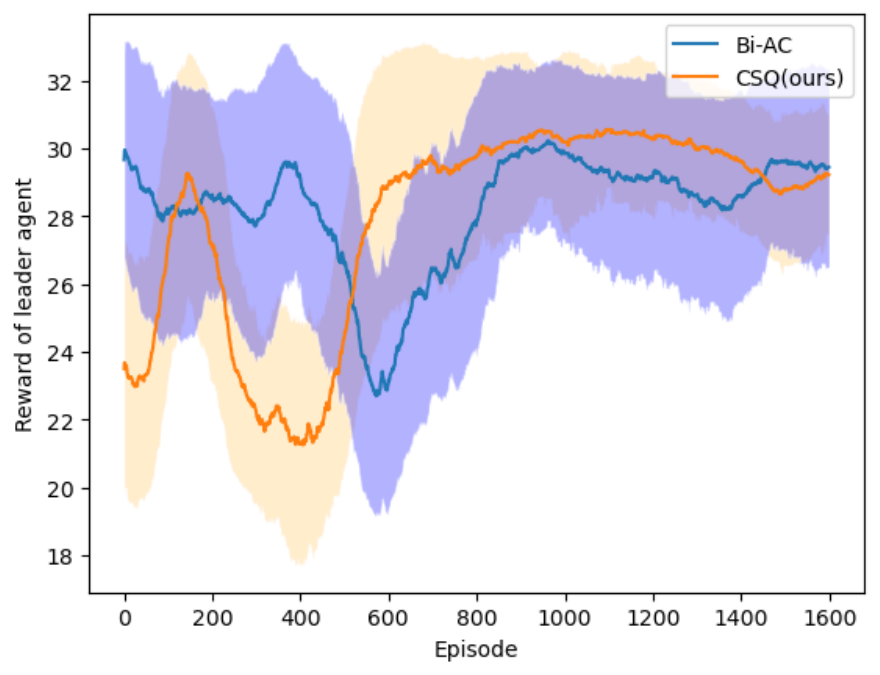} & 
    \includegraphics[width=0.3\textwidth,height=4cm]{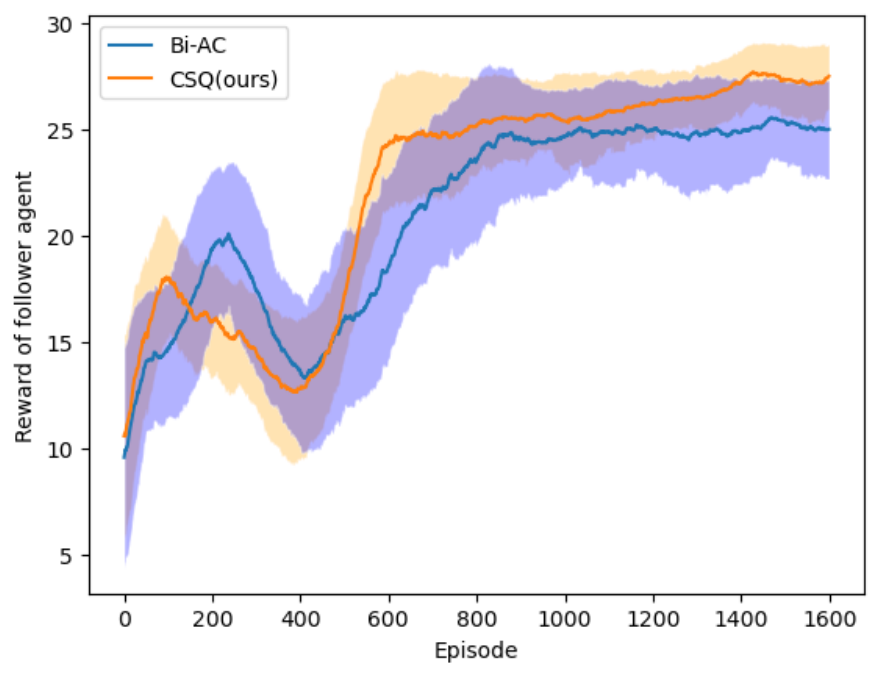} & 
    \includegraphics[width=0.3\textwidth,height=4cm]{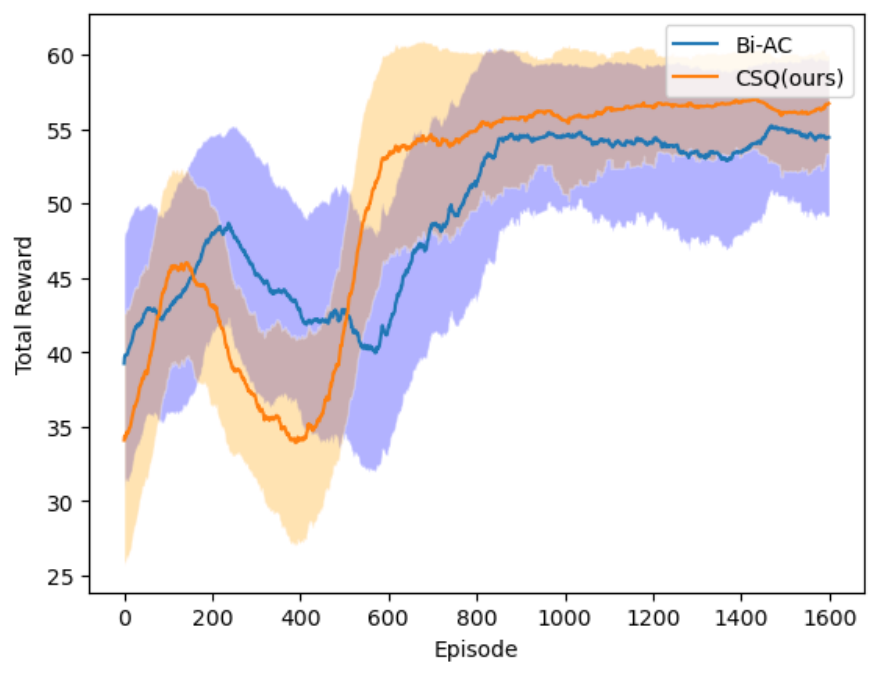}\\
    \hline
    \end{tabular}
    \caption{Table I: Reward curves for Safe Merge Environments.}
    \label{tab:merge_reward_label}
\end{table*}

\begin{table*}[]
    \centering
    \begin{tabular}{ c  c  c }
    \hline \\
    Leader & Follower & Total \\
    \hline \\
    \includegraphics[width=0.3\textwidth,height=4cm]{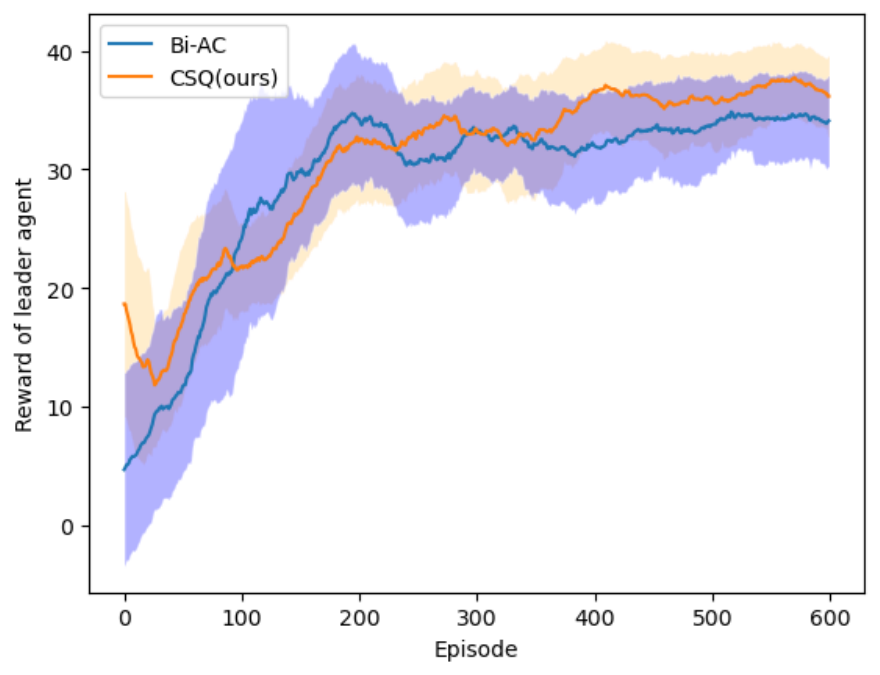} & 
    \includegraphics[width=0.3\textwidth,height=4cm]{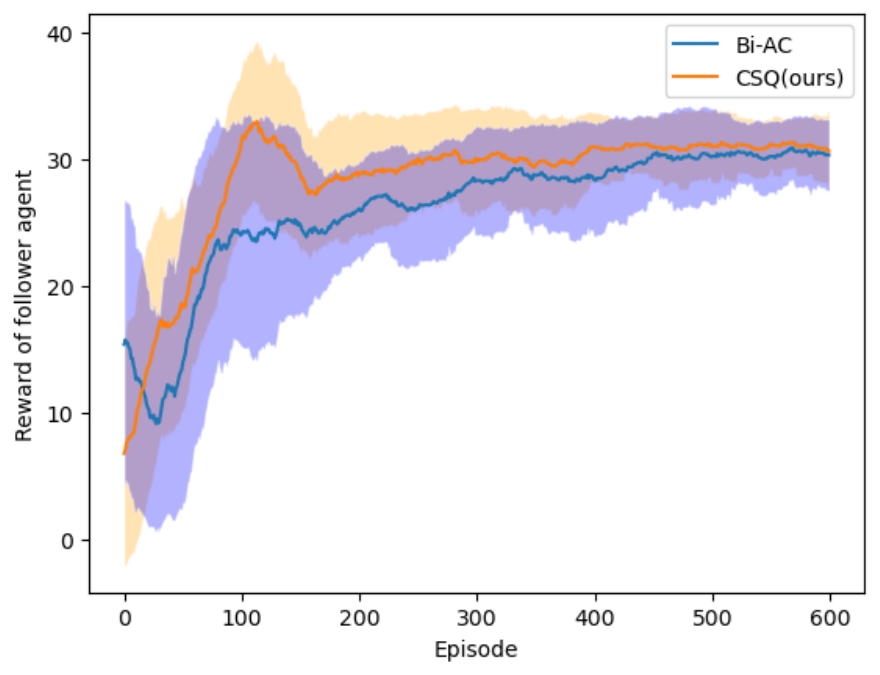} & 
    \includegraphics[width=0.3\textwidth,height=4cm]{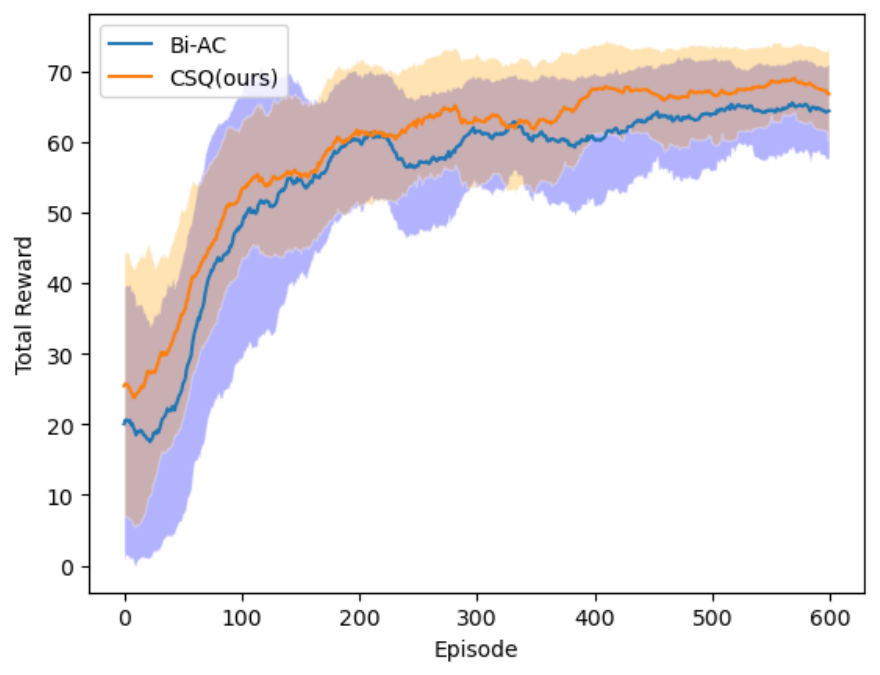} \\
    \hline
    \end{tabular}
    \caption{Table II: Reward curves of Safe Roundabout Environments.}
    \label{tab:roundabout_reward_label}
\end{table*}

\begin{figure*}[htbp]
\centering
\begin{minipage}{.5\textwidth}
  \centering
  \includegraphics[scale=0.25]{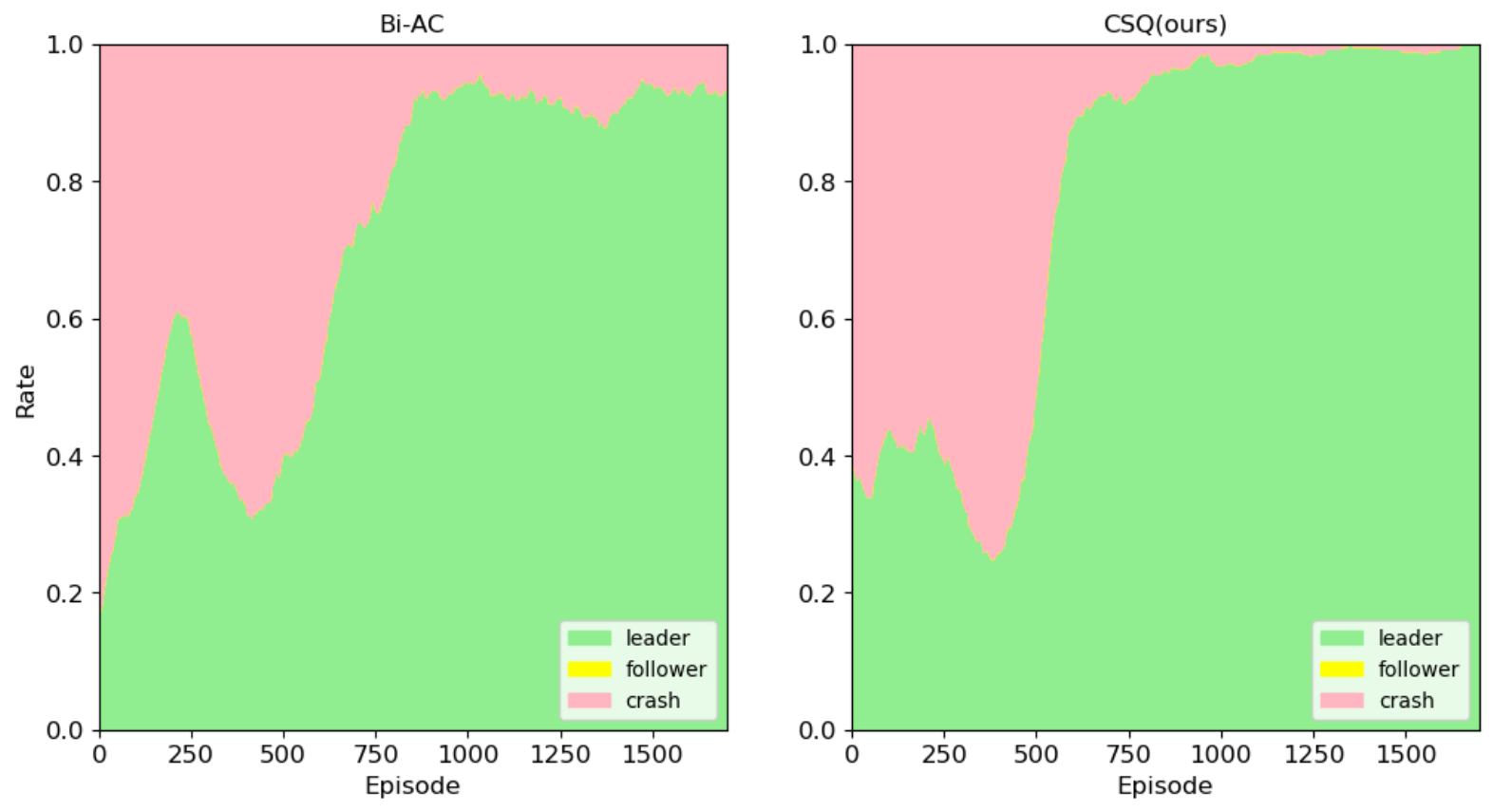}
  \caption{The training curves for Safe Merge Environments.}
  \label{Fig8.5}
\end{minipage}%
\begin{minipage}{.5\textwidth}
  \centering
  \includegraphics[scale=0.25]{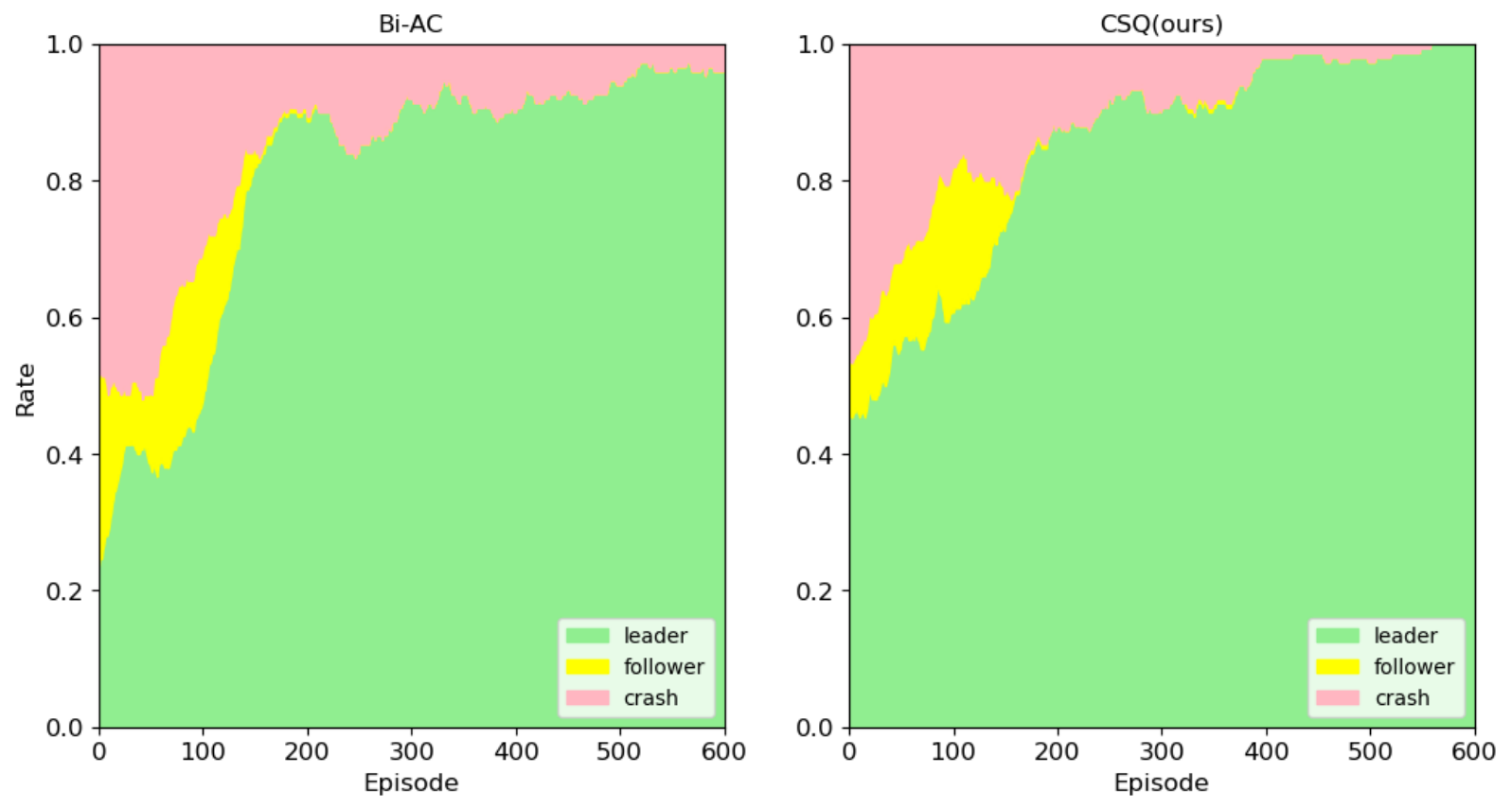}
  \caption{The training curves for Safe Roundabout Environments.}
  \label{Fig8.6}
\end{minipage}
\end{figure*}

\subsection{Safe Intersection Environments}
This scenario outlines an intersection where two roads meet at right angles, forming a four-way junction, as illustrated in Fig \ref{fig:intersection}. Vehicles approach from the north, south, east, and west, aiming to traverse the intersection and reach their destinations. Notably, this intersection lacks traffic lights, necessitating vehicles to navigate for safe passage autonomously. This experiment involves two agents: the leader, traveling from south to north, and the follower, moving from west to east. Each agent independently determines its actions and adjusts throttle variables based on observations. A bonus reward is provided for the first arrival to introduce a competitive element. However, it's crucial to emphasize that collision avoidance takes precedence, as any collisions incur a substantial penalty for both agents.

The training results are shown in Table \ref{tab:intersection_reward_label} and Fig \ref{Fig8.7}. While our CS-MADDPG algorithm achieves rewards very close to those of MACPO, it holds a slight advantage in terms of safety rates. All algorithms exhibit relatively stable rewards after convergence, but the CS-MADDPG algorithm converges faster and earlier. Under the CS-MADDPG framework, the relationship between the leader and the followers is relatively stable, with the leader predominantly taking the lead in most situations.

\begin{table*}[]
    \centering
    \begin{tabular}{ c  c  c }
    \hline \\
    Leader & Follower & Total \\
    \hline \\
    \includegraphics[width=0.3\textwidth,height=4cm]{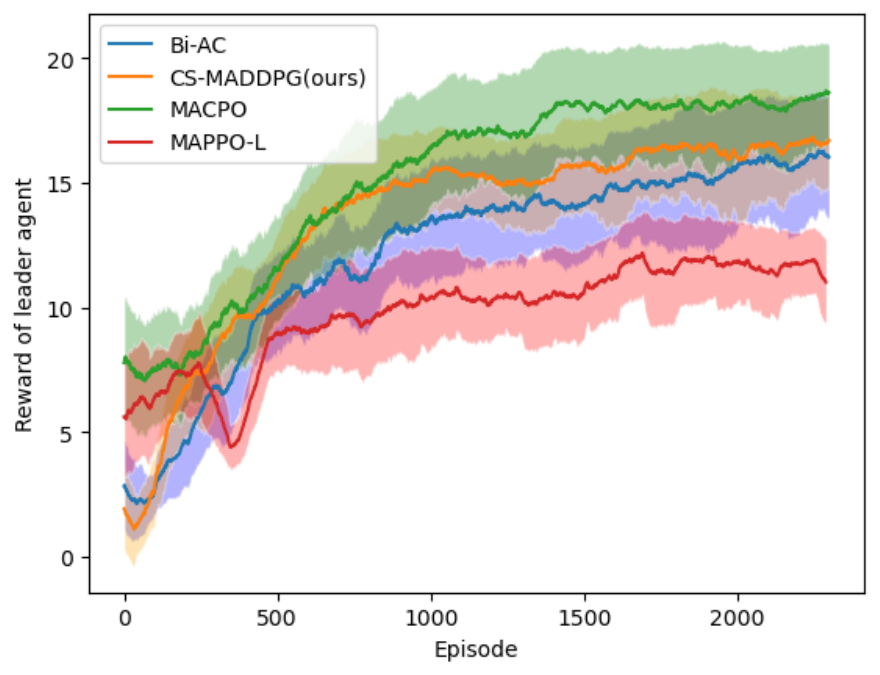} & 
    \includegraphics[width=0.3\textwidth,height=4cm]{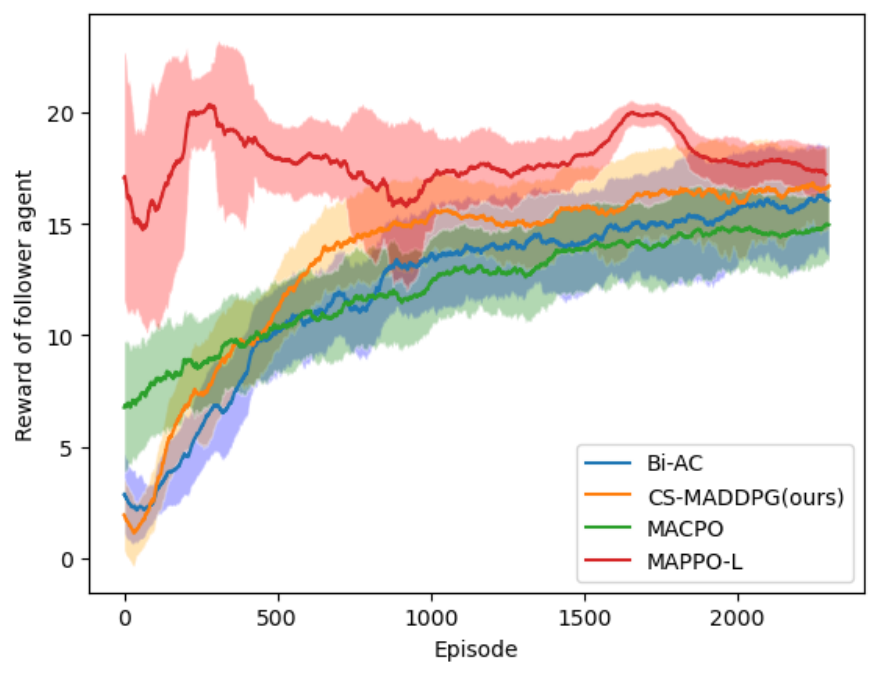} & 
    \includegraphics[width=0.3\textwidth,height=4cm]{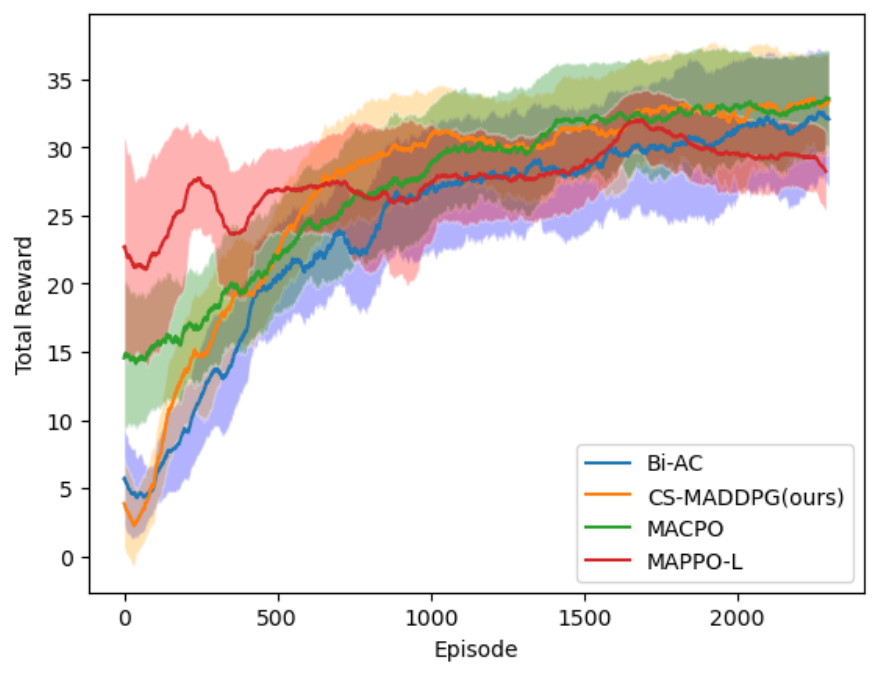} \\
    \hline
    \end{tabular}
    \caption{Table III: Reward curves of Safe Intersection Environments.}
    \label{tab:intersection_reward_label}
\end{table*}

\begin{table*}[]
    \centering
    \begin{tabular}{ c  c  c }
    \hline \\
    Leader & Follower & Total \\
    \hline \\
    \includegraphics[width=0.3\textwidth,height=4cm]{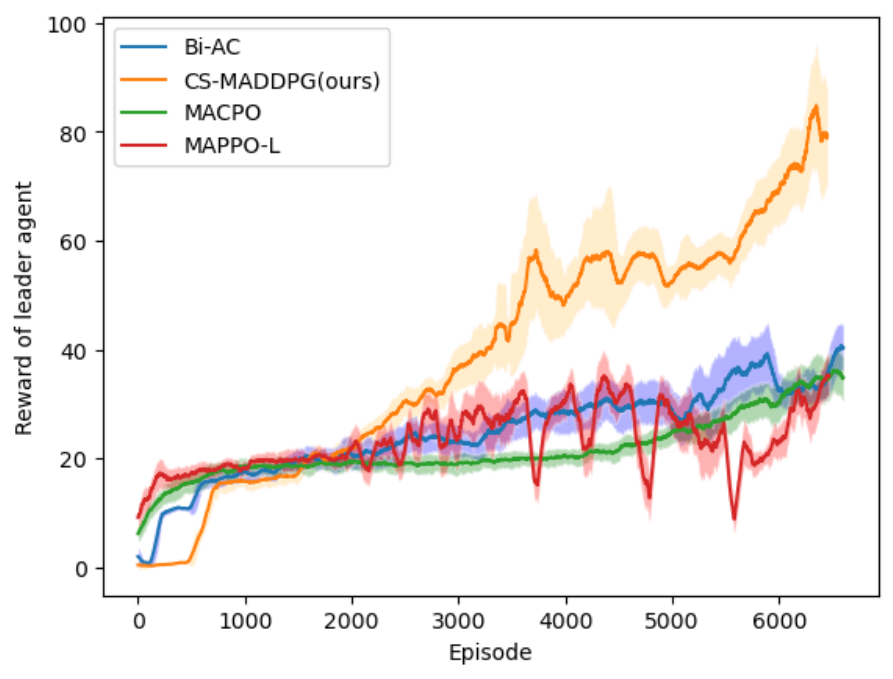} & 
    \includegraphics[width=0.3\textwidth,height=4cm]{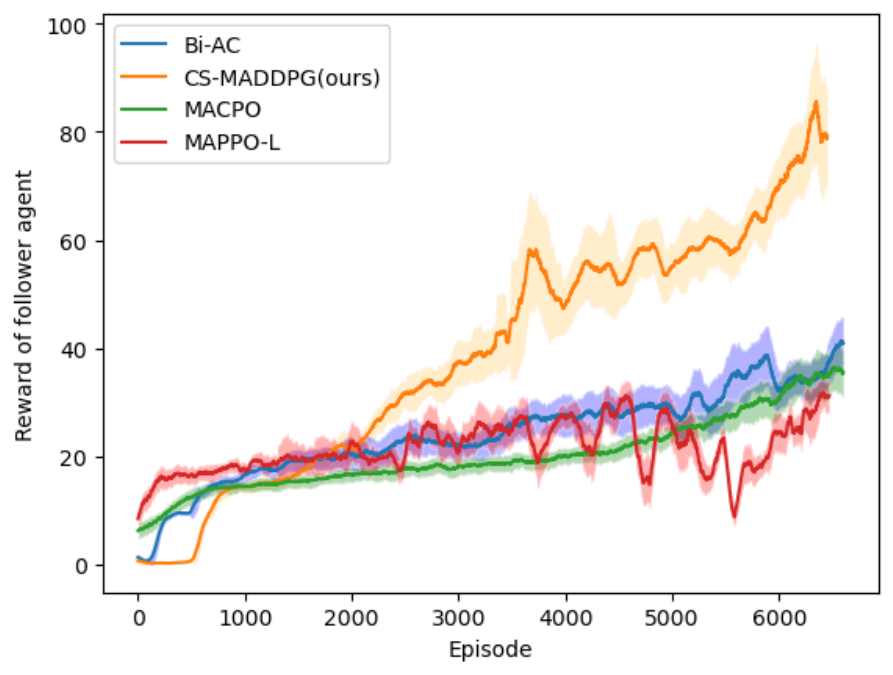} & 
    \includegraphics[width=0.3\textwidth,height=4cm]{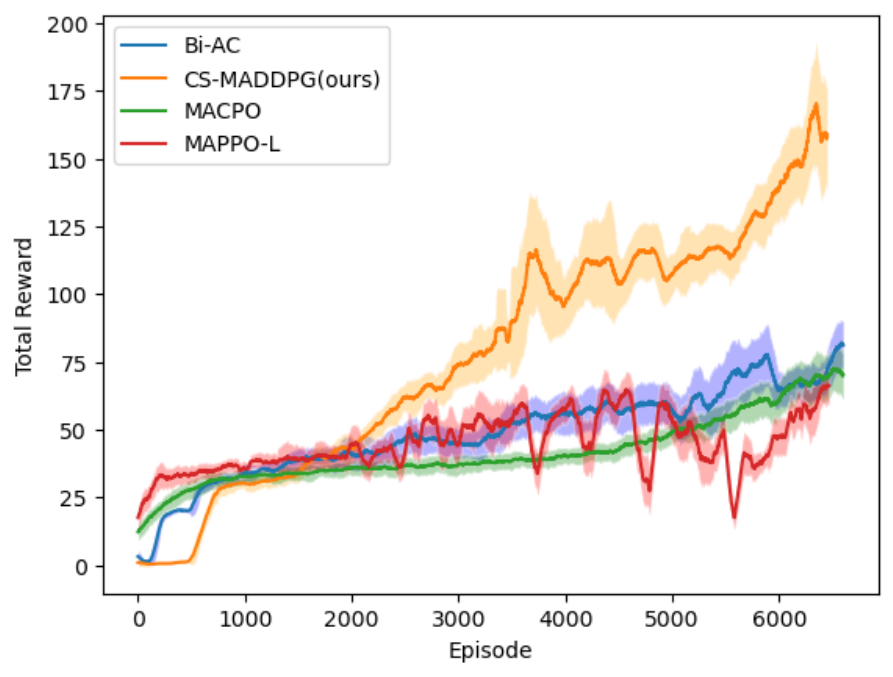} \\
    \hline
    \end{tabular}
    \caption{Table IV: Reward curves of Safe Racetrack Environments.}
    \label{tab:racetrack_reward_label}
\end{table*}

\begin{figure*}[htbp]
\centering
\begin{minipage}{.5\textwidth}
  \centering
  \includegraphics[scale=0.25]{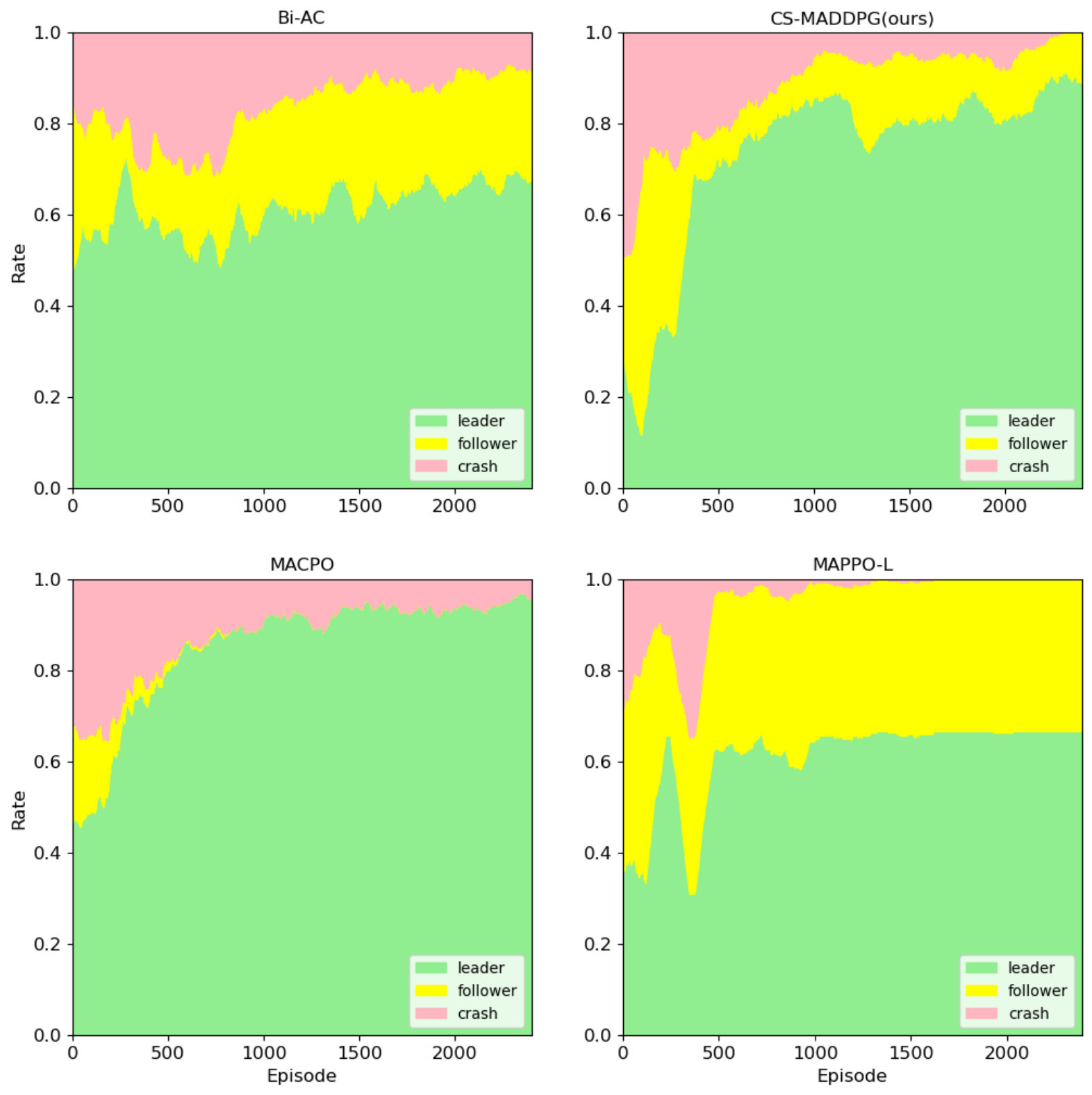}
  \caption{The training curves for Safe Intersection Environments.}
  \label{Fig8.7}
\end{minipage}%
\begin{minipage}{.5\textwidth}
  \centering
  \includegraphics[scale=0.25]{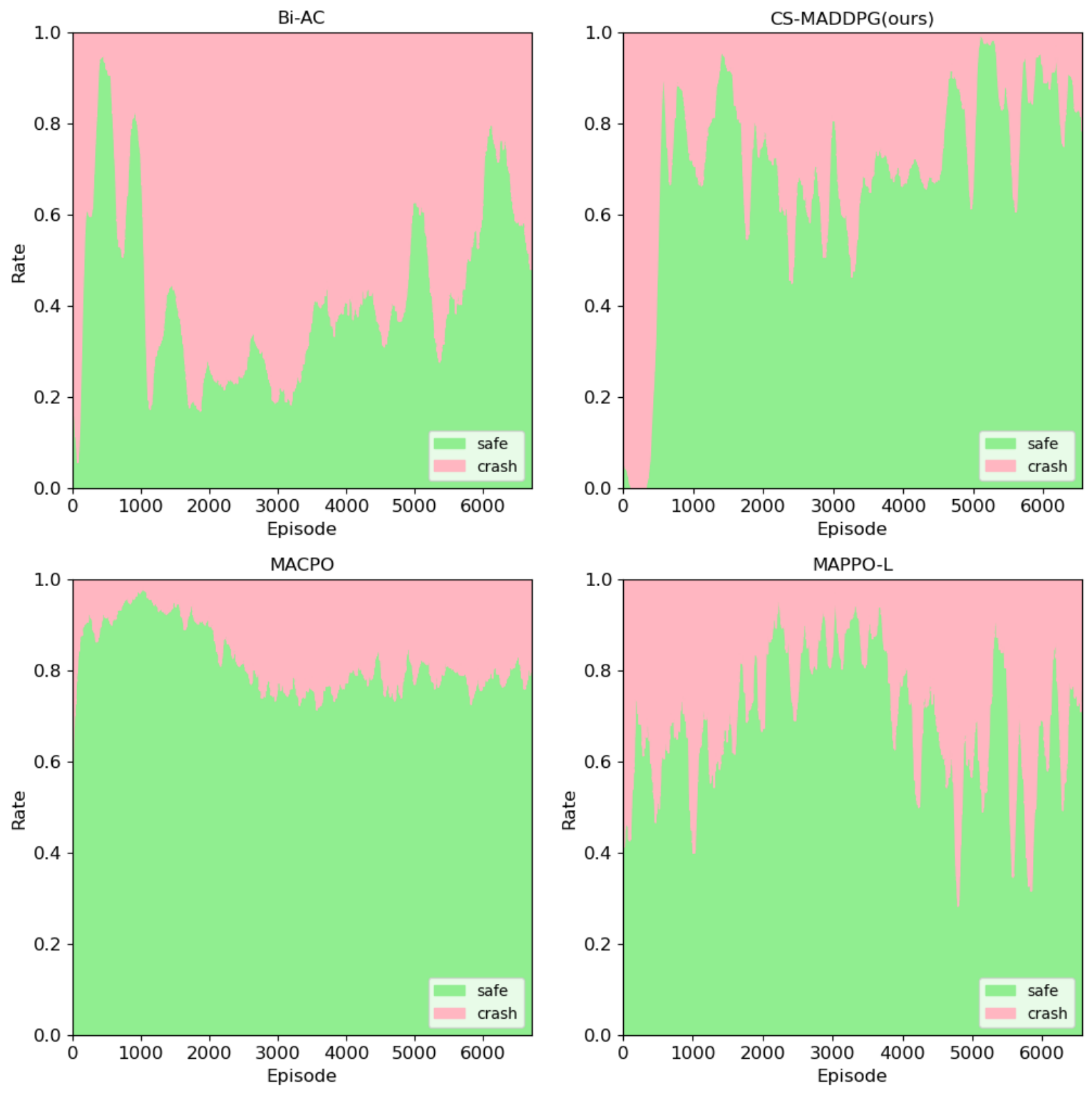}
  \caption{The training curves for Safe Racetrack Environments.}
  \label{Fig8.8}
\end{minipage}
\end{figure*}

\subsection{Safe Racetrack Environments}
The racetrack environment consists of a circular track with two lanes, as illustrated in Fig \ref{fig:racetrack}. It features straight sections and curves, rendering it more complex than previous environments. This setting involves continuous control, where agents are tasked with managing their steering angles and maintaining a constant longitudinal velocity. The leader and follower agent vehicles initiate parallel to each other at the starting point on the straight track. They aim to navigate the track while evading collisions with other vehicles and staying within the lane boundaries. Collisions become highly likely if the agents fail to control their steering angles properly, given their constant longitudinal velocity. The reward mechanism in this environment is slightly different from other environments. The vehicle earns the highest reward for staying in the center lane, and the rewards decrease as it deviates from the center line. Exiting the road results in no rewards, while collisions incur substantial penalties. The agents do not have a destination or finish line to reach but must safely travel along the lane throughout the simulation time. This setting embodies a cooperative game, necessitating the agents master the art of following the track and avoiding collisions to maximize rewards within the stipulated time frame.

The training results for the racetrack are presented in Table \ref{tab:racetrack_reward_label} and Fig \ref{Fig8.8}. Due to the complexity of the track, the baselines can not achieve a very high safety rate. Nevertheless, our algorithm demonstrates the capability to achieve rewards several times greater than those attained by other algorithms while covering longer distances. Because the follower can obtain action information from the leader and reacts more quickly to avoid collisions. 


\begin{figure}[htbp]
\centering
\includegraphics[scale=0.5]{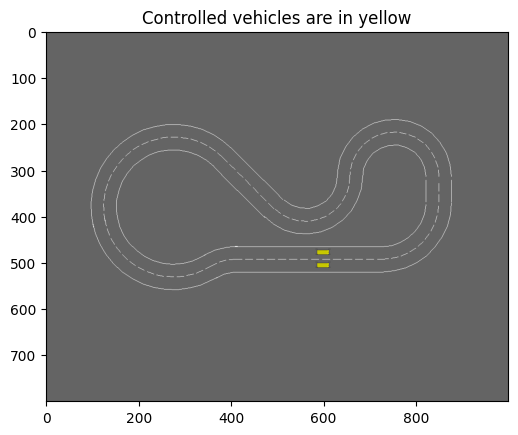}
\caption{Safe Racetrack Environments: The two vehicles need to drive safely on the track for as long as possible.}
\label{fig:racetrack}
\end{figure}

\section{Conclusion}
\label{eq:conclusion}

In this study, aiming to ensure the safety of MARL in autonomous driving, we propose a safe Stackelberg MARL method with bi-level optimization, where agents can reason about each other's behaviors when they make decisions. Derived from our method's theoretical analysis, CSQ and CS-MADDPG are developed for discrete and continuous settings. We examine the effectiveness of algorithms in our developed safe MARL autonomous driving benchmarks, and the experimental results suggest that our algorithms are better than several strong baselines, e.g., MACPO, MAPPO-L, and Bi-AC, regarding reward and safety performance. Additionally, the experimental data reveal that the leader reward, trained within the bi-level framework, typically surpasses the follower reward. This discrepancy underscores the leader's dominant position within the bi-level framework, leading to a reduction in the number of equilibrium points and enhancing their stability compared to the Nash equilibrium. This characteristic could contribute to the algorithm's effective convergence. In the future, we plan to provide more theoretical results, such as sample complexity.

 
\appendix
\section{Convergence analysis}
\label{Convergence}
In this section, we aim to examine how $Q_i^t$ for each agent i converges to a constant equilibrium $Q_i^*$. The constrained Stackelberg strategies establish the optimal value function $Q_i^*$. Our study focuses on the two-agent system: one leader agent and one follower agent. We strive to achieve the optimal value function $(Q_1^*, Q_2^*)$ as our learning objective and demonstrate the convergence of $(Q_1, Q_2)$ to $(Q_1^*, Q_2^*)$. Our convergence analysis is built on the proof of Nash-Q in \cite{nash_q} and Bi-level Actor-Critic in \cite{biac}.
\setcounter{equation}{0}
\renewcommand{\theequation}{A\arabic{equation}}
\newtheorem{assumption}{Assumption}
\newtheorem{lemma}{Lemma}
\newtheorem{definition}{Definition}
\newtheorem{theorem}{Theorem}

\begin{assumption}
 Every state $s$ in the state space $S$ and every action $a_k$ in the action space $A_k$ for agents $k = 1,2$ are visited infinitely often~\cite{nash_q}.
\end{assumption}

\begin{assumption}
 The learning rate $\alpha^t$ is defined to meet the following conditions for all $s$, $t$, $a_1$, and $a_2$~\cite{nash_q}:
\\
1. $0 \leq \alpha^t(s,a_1,a_2)<1, \sum^{\infty}_{t=0}\alpha^t(s, a_1, a_2) = \infty, 
\\ \sum^{\infty}_{t=0}[\alpha^t(s, a_1, a_2)]^2 < \infty$, and the latter two hold uniformly and with probability one.
\\
2. $\alpha^t(s, a_1, a_2) = 0$, if $(s, a_1, a_2) \neq (s^t, a^t_1, a^t_2)$.
\end{assumption}

\begin{assumption}
For every state $s$ in the state space $S$, the safe set of action space $A_k^{safe}(s)$ is not empty for agents $k = 1,2$.
\end{assumption}

\begin{lemma}
(Szepesvari and Littman (1999), Corollary 5).
Assume that $\alpha^t$
satisfies Assumption 2, and the mapping $P^t: \mathbb{Q} \to \mathbb{Q}$ satisfies the following condition: there exists a number $0 < \gamma < 1$ and a sequence $\lambda^t \geq 0$ converging into zero with probability one such that $||P^tQ - P^tQ^*|| \leq \gamma||Q - Q^*|| + \lambda^t$ for all $Q \in \mathbb{Q}$ and $Q^* = E[P^tQ^*]$, then the iteration defined by $Q^{t+1} = (1 - \alpha)Q + \alpha^t[P^tQ^t]$ converges to $Q^*$ with probability one \cite{nash_q}.

\end{lemma}

\begin{definition}
Let $Q = (Q_1, Q_2)$ where $Q_1 \in \mathbb{Q}_1, Q_2 \in \mathbb{Q}_2$, and $\mathbb{Q} = \mathbb{Q}_1 \times \mathbb{Q}_2$, $P^t:\mathbb{Q} \to \mathbb{Q}$ is a mapping on the complete metric space $\mathbb{Q}$ into $\mathbb{Q}$, $P^tQ = (P^tQ_1, P^tQ_2)$, where
\begin{equation}
    \begin{aligned}
        P^tQ_i(s,a_1,a_2) &= r^t_i(s,a_1,a_2) + \gamma Q_i^{CS}(s'),
        \\
        & for \quad i = 1,2.
    \end{aligned}
\end{equation}
where s' is the state at time t+1, and $Q_i^{CS}(s') = Q_i(s', \pi_1(s'), \pi_2(s'))$ is a constrained Stackelberg equilibrium solution for the game $(Q_1(s'), Q_2(s'))$ defined in Constrained Stackelberg Q-learning, as shown in Equations (\ref{p10}), (\ref{p11}), (\ref{p12}), and (\ref{p13}).
\end{definition}

\begin{lemma}
For an 2-player stochastic or deterministic game, $E[P^tQ^*] = Q^*$, where $Q^* = (Q_1^*,Q_2^*)$.
\end{lemma}

\begin{proof}
    \begin{align*}
    &Q_i^*(s,a_1,a_2)
    =r_i(s,a_1,a_2) + \gamma \sum p(s'|s,a_1,a_2)Q_i^{*,{CS}}(s')
    \\&
    =r_i(s,a_1,a_2) + \gamma \sum p(s'|s,a_1,a_2)Q_i^*(s',\pi_1(s'), \pi_2(s'))
    \\&
    =\sum p(s'|s,a_1,a_2)(\gamma Q_i^*(s',\pi_1(s'),\pi_2(s')) + r_i(s,a_1,a_2))
    \\&
    =E[P_i^tQ_i^*(s, a1, a2)],
    \\&
    \Rightarrow E[P^tQ^*] = Q^*.
    \end{align*}
\end{proof}

\begin{assumption}
    For every t and s, $(Q^t_1(s), Q^t_2(s))$ has a global optimal point, which is the constrained Stackelberg equilibrium point, and agents’ payoffs at this optimal point are selected with 100\% certainty since all possible combinations of actions have been explored.
\end{assumption}
\hfill 
\begin{definition}
    The joint strategy $(\sigma_1,\sigma_2)$ is the global optimal point of constrained Stackelberg game $(Q_1(s), Q_2(s))$.
    \\
    \\
    For the sake of simplicity in notation, we will use $\sigma_i$ to denote $\pi_i(s)$ and $\hat{\sigma}_i$ to represent $\hat{\pi}_i(s)$.
    \\
    For agent 1 (leader agent), its payoff at the optimal point is higher than the payoff obtained by any other joint strategies:
    \\
    $\sigma_1\sigma_2Q^j_1(s) \geq 
    \hat{\sigma}_1 \hat{\sigma}_2Q^j_1(s)$ for all $\hat{\sigma}_1,\hat{\sigma}_2 \in \sigma(A_1^{safe}(s)),\sigma(A_2^{safe}(s))$.
    \\
    For agent 2 (follower agent), its payoff at the optimal point is higher than the payoff obtained by any other follower strategies:
    \\
    $\sigma_1\sigma_2Q^j_2(s) \geq \sigma_1\hat{\sigma}_2Q^j_2(s)$ for all $\hat{\sigma}_2 \in \sigma(A_2^{safe}(s))$.
    \\
\end{definition}

\begin{definition}
(Hu and Wellman (2003)) \cite{nash_q}) For $Q, \hat{Q} \in \mathbb{Q}$, define:
    \\
    \begin{equation}
        \begin{aligned}
            &||Q - \hat{Q}||
            \\
            & \equiv \max_{j} \max_{s}||Q_j(s) - \hat{Q}_j(s)||
            \\
            & \equiv \max_{j} \max_{s} \max_{a_1,a_2}||Q_j(s,a_1,a_2) - \hat{Q}_j(s,a_1,a_2)||.
        \end{aligned}
    \end{equation}
\end{definition}

\begin{lemma}
$||P^tQ - P^t\hat{Q}|| \leq \gamma  ||Q - \hat{Q}||$ for all $Q \in \mathbb{Q}$.
\end{lemma} 

\begin{proof}
    \begin{align*}
        &||P^tQ - P^t \hat{Q}|| 
         = \max_{j}||P^tQ_j - P^t \hat{Q}_j||
        \\
        & = \max_{j}\max_{s}|\gamma\pi_1(s)\pi_2(s)Q_j(s)-\gamma\hat{\pi}_1(s)\hat{\pi}_2(s)\hat{Q}_j(s)|
        \\
        & = \max_{j}\gamma|\pi_1(s)\pi_2(s)Q_j(s)-\hat{\pi}_1(s)\hat{\pi}_2(s)\hat{Q}_j(s)|.
    \end{align*}

    We proceed to prove that
    \begin{equation}
        |\pi_1(s)\pi_2(s)Q_j(s)-\hat{\pi}_1(s)\hat{\pi}_2(s)\hat{Q}_j(s)| \leq |Q_j(s) - \hat{Q}_j(s)|.
    \end{equation}
    
    After simplifying the notation:
    \begin{equation}
        |\sigma_1\sigma_2Q_j(s)-\hat{\sigma}_1\hat{\sigma}_2\hat{Q}_j(s)| \leq |Q_j(s) - \hat{Q}_j(s)|. 
    \end{equation}
    
    \textbf{1)Leader agent (Agent 1):}
    
    if $\sigma_1\sigma_2Q_1(s) \geq \hat{\sigma}_1 \hat{\sigma}_2\hat{Q}_1(s)$, we have:
    
    \begin{equation}
        \begin{aligned}
            &\sigma_1\sigma_2Q_1(s) - \hat{\sigma}_1 \hat{\sigma}_2 \hat{Q}_1(s)
            \\
            \leq & \sigma_1\sigma_2Q_1(s) - \sigma_1 \sigma_2 \hat{Q}_1(s)
            \\
            = & \sum \sigma_1(a_1)\sigma_2(a_2) (Q_1(s,a_1,a_2)- \hat{Q}_1(s,a_1,a_2))
            \\
            \leq & \max_{s} \sum \sigma_1(a_1)\sigma_2(a_2) (Q_1(s,a_1,a_2)- \hat{Q}_1(s,a_1,a_2))
            \\
            = & \sum \sigma_1(a_1)\sigma_2(a_2) ||Q_1(s,a_1,a_2)- \hat{Q}_1(s,a_1,a_2))||
            \\
            = & ||Q_1(s)- \hat{Q}_1(s))||.
        \end{aligned}
    \end{equation}
    
    If $\sigma_1\sigma_2Q_1(s) \leq \hat{\sigma}_1 \hat{\sigma}_2\hat{Q}_1(s)$, the proof is similar to the above, thus
    
    \begin{equation}
        ||P^tQ_1 - P^t \hat{Q}_1|| \leq \gamma ||Q_1(s)- \hat{Q}_1(s))||.
    \end{equation}
    
    It means $Q_1$ will converge to a fixed point. Once $Q_1$ is converged, $\sigma_1$ will also be fixed for all states $s$, i.e., $\pi_1(s) = \hat{\pi}_1(s)$ ($\sigma_1 = \hat{\sigma}_1$) for all $Q, \hat{Q} \in \mathbb{Q}$.

    \textbf{2)Follower agent (Agent 2):}
    if $\sigma_1\sigma_2Q_2(s) \geq \hat{\sigma}_1 \hat{\sigma}_2\hat{Q}_2(s)$.
    
    \begin{equation}
        \begin{aligned}
            & \sigma_1\sigma_2Q_2(s) - \hat{\sigma}_1 \hat{\sigma}_2 \hat{Q}_2(s)
            \\
            = & \sigma_1\sigma_2Q_2(s) - \sigma_1 \hat{\sigma}_2 \hat{Q}_2(s)
            \\
            \leq & \sigma_1\sigma_2Q_2(s) - \sigma_1 \sigma_2 \hat{Q}_2(s)
            \\
            = & \sum \sigma_1(a_1)\sigma_2(a_2) (Q_2(s,a_1,a_2)- \hat{Q}_2(s,a_1,a_2))
            \\
            \leq & \max_{s} \sum \sigma_1(a_1)\sigma_2(a_2) ||Q_2(s,a_1,a_2)- \hat{Q}_2(s,a_1,a_2))||
            \\
            = & \sum \sigma_1(a_1)\sigma_2(a_2) ||Q_2(s,a_1,a_2)- \hat{Q}_2(s,a_1,a_2))||
            \\
            = & ||Q_2(s)- \hat{Q}_2(s))||.
        \end{aligned}
    \end{equation}
    
    If $\sigma_1\sigma_2Q_2(s) \leq \hat{\sigma}_1 \hat{\sigma}_2\hat{Q}_2(s)$, the proof is similar to the above, thus
    
    \begin{equation}
        ||P^tQ_2 - P^t \hat{Q}_2|| \leq \gamma ||Q_2(s)- \hat{Q}_2(s))||.
    \end{equation}
\end{proof}

\begin{theorem}
    Under assumptions 1 - 3, the sequence $Q^t = (Q^t_1
    , Q^t_2)$, updated by

    \begin{equation}
    \begin{aligned} 
         & a_1' = \mathop{\arg\max}_{a_1} Q_1(s', a_1, a_2'),
        \\
        & \quad \textrm{s.t.} \quad G_{1}(s', a_1, a_2') \leq d_{1},
        \\
        & \qquad a_2' = \mathop{\arg\max}_{a_2} Q_2(s', a_1, a_2),
        \\
        & \qquad \quad \textrm{s.t.} \quad G_{2}(s', a_1, a_2) \leq d_{2}. 
    \end{aligned}
\end{equation}
\begin{subequations}
\begin{equation}
    \begin{aligned}
        & \phi_1 \gets \phi_1 - \alpha _1 \bigtriangledown_{\phi_1}(Q_1(s,a_1,a_2) - r_1
        \\
        & \qquad - \gamma (1-d)Q_1^{targ}(s',a_1',a_2')),
    \end{aligned}
\end{equation}
\begin{equation}
    \begin{aligned}
        & \phi_2 \gets \phi_2 - \alpha_2 \bigtriangledown_{\phi_2}(Q_2(s,a_1,a_2)  - r_2
        \\
        & \qquad - \gamma (1-d)Q_2^{targ}(s',a_1',a_2')),
    \end{aligned}
\end{equation}

\begin{equation}
    \begin{aligned}
        & \zeta_1 \gets \zeta_1 - \beta_1 \bigtriangledown_{\zeta_1}(G_{1}(s,a_1,a_2) - c_1
        \\
        & \qquad - \gamma (1-d)G_{1}^{targ}(s',a_1',a_2')),
    \end{aligned}
\end{equation}

\begin{equation}
    \begin{aligned}
        & \zeta_2 \gets \zeta_2 - \beta_2 \bigtriangledown_{\zeta_2}(G_{2}(s,a_1,a_2) - c_2 
        \\
        & \qquad- \gamma (1-d)G_{2}^{targ}(s',a_1',a_2')).
    \end{aligned}
\end{equation}

\end{subequations}
in Constrained Stackelberg Q-learning leads to the convergence of the Q-values to a fixed value denoted as $Q^* = (Q_1^*, Q_2^*)$. This convergence is established through two key observations.

First, by applying Lemma 3, it is shown that the update operator $P^t$ is a contraction operator. This property ensures that the Q-values progressively approach a fixed point as the iterations proceed.

Second, the fixed point condition, represented by the equation $E[P^tQ^*] = Q^*$, is established using Lemma 2. This condition indicates that the expected value of applying the update operator to the fixed Q-value results in the same Q-value.

Combining these observations with the insights from Lemma 1, we conclude that the Q-values will converge to $Q^*$ with probability 1. In other words, as the algorithm continues to iterate, the Q-values approach the optimal solution, ensuring the convergence of the reinforcement learning process.
\end{theorem}

\normalem 

\bibliography{main}

\begin{thebibliography}{}

\end{thebibliography}


\begin{thebibliography}{10}

\bibitem{CPO}
Joshua Achiam, David Held, Aviv Tamar, and Pieter Abbeel.
\newblock Constrained policy optimization.
\newblock In {\em International conference on machine learning}, pages 22--31. PMLR, 2017.

\bibitem{altman-constrainedMDP}
E.~Altman.
\newblock {\em Constrained Markov Decision Processes}.
\newblock Chapman and Hall, 1999.

\bibitem{brockman2016openai}
Greg Brockman, Vicki Cheung, Ludwig Pettersson, Jonas Schneider, John Schulman, Jie Tang, and Wojciech Zaremba.
\newblock Openai gym.
\newblock {\em arXiv preprint arXiv:1606.01540}, 2016.

\bibitem{cai2021safe}
Zhiyuan Cai, Huanhui Cao, Wenjie Lu, Lin Zhang, and Hao Xiong.
\newblock Safe multi-agent reinforcement learning through decentralized multiple control barrier functions.
\newblock {\em arXiv preprint arXiv:2103.12553}, 2021.

\bibitem{ding2023provably}
Dongsheng Ding, Xiaohan Wei, Zhuoran Yang, Zhaoran Wang, and Mihailo Jovanovic.
\newblock Provably efficient generalized lagrangian policy optimization for safe multi-agent reinforcement learning.
\newblock In {\em Learning for Dynamics and Control Conference}, pages 315--332. PMLR, 2023.

\bibitem{elsayed2021safe}
Ingy ElSayed-Aly, Suda Bharadwaj, Christopher Amato, R{\"u}diger Ehlers, Ufuk Topcu, and Lu~Feng.
\newblock Safe multi-agent reinforcement learning via shielding.
\newblock In {\em Proceedings of the 20th International Conference on Autonomous Agents and MultiAgent Systems}, pages 483--491, 2021.

\bibitem{JMLR:v16:garcia15a}
Javier Garc{{\'i}}a, Fern, and o~Fern{{\'a}}ndez.
\newblock A comprehensive survey on safe reinforcement learning.
\newblock {\em Journal of Machine Learning Research}, 16(42):1437--1480, 2015.

\bibitem{gu2022constrained}
Shangding Gu, Guang Chen, Lijun Zhang, Jing Hou, Yingbai Hu, and Alois Knoll.
\newblock Constrained reinforcement learning for vehicle motion planning with topological reachability analysis.
\newblock {\em Robotics}, 11(4):81, 2022.

\bibitem{gu2024safe}
Shangding Gu, Dianye Huang, Muning Wen, Guang Chen, and Alois Knoll.
\newblock Safe multiagent learning with soft constrained policy optimization in real robot control.
\newblock {\em IEEE Transactions on Industrial Informatics}, 2024.

\bibitem{gu2023human}
Shangding Gu, Alap Kshirsagar, Yali Du, Guang Chen, Jan Peters, and Alois Knoll.
\newblock A human-centered safe robot reinforcement learning framework with interactive behaviors.
\newblock {\em Frontiers in Neurorobotics}, 17, 2023.

\bibitem{GU2023103905}
Shangding Gu, Jakub~Grudzien Kuba, Yuanpei Chen, Yali Du, Long Yang, Alois Knoll, and Yaodong Yang.
\newblock Safe multi-agent reinforcement learning for multi-robot control.
\newblock {\em Artificial Intelligence}, 319:103905, 2023.

\bibitem{gu2022review}
Shangding Gu, Long Yang, Yali Du, Guang Chen, Florian Walter, Jun Wang, Yaodong Yang, and Alois Knoll.
\newblock A review of safe reinforcement learning: Methods, theory and applications.
\newblock {\em arXiv preprint arXiv:2205.10330}, 2022.

\bibitem{nash_q}
Junling Hu and Michael~P. Wellman.
\newblock Nash q-learning for general-sum stochastic games.
\newblock {\em J. Mach. Learn. Res.}, 4(null):1039–1069, dec 2003.

\bibitem{huang2022robust}
Peide Huang, Mengdi Xu, Fei Fang, and Ding Zhao.
\newblock Robust reinforcement learning as a stackelberg game via adaptively-regularized adversarial training, 2022.

\bibitem{kuba2021settling}
Jakub~Grudzien Kuba, Muning Wen, Linghui Meng, Shangding Gu, Haifeng Zhang, David Mguni, Jun Wang, and Yaodong Yang.
\newblock Settling the variance of multi-agent policy gradients.
\newblock {\em Advances in Neural Information Processing Systems}, 34:13458--13470, 2021.

\bibitem{kuba2022trust}
JG~Kuba, R~Chen, M~Wen, Y~Wen, F~Sun, J~Wang, and Y~Yang.
\newblock Trust region policy optimisation in multi-agent reinforcement learning.
\newblock In {\em ICLR 2022-10th International Conference on Learning Representations}, page 1046. The International Conference on Learning Representations (ICLR), 2022.

\bibitem{highway-env}
Edouard Leurent.
\newblock An environment for autonomous driving decision-making.
\newblock \url{https://github.com/eleurent/highway-env}, 2018.

\bibitem{liu2021cmix}
Chenyi Liu, Nan Geng, Vaneet Aggarwal, Tian Lan, Yuan Yang, and Mingwei Xu.
\newblock Cmix: Deep multi-agent reinforcement learning with peak and average constraints.
\newblock In {\em Machine Learning and Knowledge Discovery in Databases. Research Track: European Conference, ECML PKDD 2021, Bilbao, Spain, September 13--17, 2021, Proceedings, Part I 21}, pages 157--173. Springer, 2021.

\bibitem{MADDPG}
Ryan Lowe, Yi~I Wu, Aviv Tamar, Jean Harb, OpenAI Pieter~Abbeel, and Igor Mordatch.
\newblock Multi-agent actor-critic for mixed cooperative-competitive environments.
\newblock {\em Advances in neural information processing systems}, 30, 2017.

\bibitem{lu2021decentralized}
Songtao Lu, Kaiqing Zhang, Tianyi Chen, Tamer Ba{\c{s}}ar, and Lior Horesh.
\newblock Decentralized policy gradient descent ascent for safe multi-agent reinforcement learning.
\newblock In {\em Proceedings of the AAAI Conference on Artificial Intelligence}, volume~35, pages 8767--8775, 2021.

\bibitem{melcer2022shield}
Daniel Melcer, Christopher Amato, and Stavros Tripakis.
\newblock Shield decentralization for safe multi-agent reinforcement learning.
\newblock {\em Advances in Neural Information Processing Systems}, 35:13367--13379, 2022.

\bibitem{bimarkov}
Linghui Meng, Jingqing Ruan, Dengpeng Xing, and Bo~Xu.
\newblock Learning in bi-level markov games.
\newblock In {\em 2022 International Joint Conference on Neural Networks (IJCNN)}, pages 1--8, 2022.

\bibitem{naveiro2019gradient}
Roi Naveiro and David~Ríos Insua.
\newblock Gradient methods for solving stackelberg games, 2019.

\bibitem{vehicelemodel}
Philip Polack, Florent Altché, Brigitte d'Andréa Novel, and Arnaud de~La~Fortelle.
\newblock The kinematic bicycle model: A consistent model for planning feasible trajectories for autonomous vehicles?
\newblock In {\em 2017 IEEE Intelligent Vehicles Symposium (IV)}, pages 812--818, 2017.

\bibitem{rashid2020monotonic}
Tabish Rashid, Mikayel Samvelyan, Christian~Schroeder De~Witt, Gregory Farquhar, Jakob Foerster, and Shimon Whiteson.
\newblock Monotonic value function factorisation for deep multi-agent reinforcement learning.
\newblock {\em The Journal of Machine Learning Research}, 21(1):7234--7284, 2020.

\bibitem{schulman2015trust}
John Schulman, Sergey Levine, Pieter Abbeel, Michael Jordan, and Philipp Moritz.
\newblock Trust region policy optimization.
\newblock In {\em International conference on machine learning}, pages 1889--1897. PMLR, 2015.

\bibitem{sherali1984multiple}
Hanif~D Sherali.
\newblock A multiple leader stackelberg model and analysis.
\newblock {\em Operations Research}, 32(2):390--404, 1984.

\bibitem{sunehag2018value}
Peter Sunehag, Guy Lever, Audrunas Gruslys, Wojciech~Marian Czarnecki, Vinicius Zambaldi, Max Jaderberg, Marc Lanctot, Nicolas Sonnerat, Joel~Z Leibo, Karl Tuyls, et~al.
\newblock Value-decomposition networks for cooperative multi-agent learning based on team reward.
\newblock In {\em Proceedings of the 17th International Conference on Autonomous Agents and MultiAgent Systems}, pages 2085--2087, 2018.

\bibitem{tan1993multi}
Ming Tan.
\newblock Multi-agent reinforcement learning: Independent vs. cooperative agents.
\newblock In {\em Proceedings of the tenth international conference on machine learning}, pages 330--337, 1993.

\bibitem{DBLP:journals/corr/abs-1805-11074}
Chen Tessler, Daniel~J. Mankowitz, and Shie Mannor.
\newblock Reward constrained policy optimization.
\newblock In {\em International Conference on Learning Representations}, 2019.

\bibitem{Uther2010}
William Uther.
\newblock {\em Markov Decision Processes}.
\newblock Springer US, Boston, MA, 2010.

\bibitem{stackelberg}
Heinrich Von~Stackelberg.
\newblock {\em Market structure and equilibrium}.
\newblock Springer Science \& Business Media, 2010.

\bibitem{yu2022surprising}
Chao Yu, Akash Velu, Eugene Vinitsky, Jiaxuan Gao, Yu~Wang, Alexandre Bayen, and Yi~Wu.
\newblock The surprising effectiveness of ppo in cooperative multi-agent games.
\newblock {\em Advances in Neural Information Processing Systems}, 35:24611--24624, 2022.

\bibitem{biac}
Haifeng Zhang, Weizhe Chen, Zeren Huang, Minne Li, Yaodong Yang, Weinan Zhang, and Jun Wang.
\newblock Bi-level actor-critic for multi-agent coordination.
\newblock In {\em Proceedings of the AAAI Conference on Artificial Intelligence}, volume~34, pages 7325--7332, 2020.

\bibitem{zhang2021multi}
Kaiqing Zhang, Zhuoran Yang, and Tamer Ba{\c{s}}ar.
\newblock Multi-agent reinforcement learning: A selective overview of theories and algorithms.
\newblock {\em Handbook of reinforcement learning and control}, pages 321--384, 2021.

\bibitem{zhang2023spatial}
Zhili Zhang, Songyang Han, Jiangwei Wang, and Fei Miao.
\newblock Spatial-temporal-aware safe multi-agent reinforcement learning of connected autonomous vehicles in challenging scenarios.
\newblock In {\em 2023 IEEE International Conference on Robotics and Automation (ICRA)}, pages 5574--5580. IEEE, 2023.

\bibitem{zheng2022stackelberg}
Liyuan Zheng, Tanner Fiez, Zane Alumbaugh, Benjamin Chasnov, and Lillian~J Ratliff.
\newblock Stackelberg actor-critic: Game-theoretic reinforcement learning algorithms.
\newblock In {\em Proceedings of the AAAI conference on artificial intelligence}, volume~36, pages 9217--9224, 2022.

\end{thebibliography}
\bibliographystyle{plain}

\end{document}